\documentclass[11pt,oneside,reqno]{amsart}

\usepackage[T1]{fontenc}
\usepackage{amssymb}
\usepackage{algpseudocode}
\usepackage{algorithm}
\usepackage{verbatim}
\usepackage{graphicx}
\usepackage{placeins}
\usepackage{listings}
\usepackage{xcolor}
\usepackage{subcaption}
\usepackage[noadjust]{cite}
\RequirePackage{dsfont} 
 \scrollmode
\allowdisplaybreaks
\usepackage{graphicx}
\usepackage{epstopdf}
\usepackage[hidelinks]{hyperref} 
\usepackage{mathtools}
\usepackage{float}

\usepackage{amsmath}
\usepackage{tikz}
\usepackage{enumerate}
\usepackage{xcolor}


\definecolor{codegreen}{rgb}{0,0.6,0}
\definecolor{codegray}{rgb}{0.5,0.5,0.5}
\definecolor{codepurple}{rgb}{0.58,0,0.82}
\definecolor{backcolour}{rgb}{0.95,0.95,0.92}

\lstdefinestyle{list_style}{
  backgroundcolor=\color{backcolour}, commentstyle=\color{codegreen},
  keywordstyle=\color{magenta},
  numberstyle=\tiny\color{codegray},
  stringstyle=\color{codepurple},
  basicstyle=\ttfamily\footnotesize,
  breakatwhitespace=false,         
  breaklines=true,                 
  captionpos=b,                    
  keepspaces=true,                 
  numbers=left,                    
  numbersep=5pt,                  
  showspaces=false,                
  showstringspaces=false,
  showtabs=false,                  
  tabsize=2
}

\newcommand{\xdasharrow}[2][->]{
\tikz[baseline=-\the\dimexpr\fontdimen22\textfont2\relax]{
\node[anchor=south,font=\scriptsize, inner ysep=1.5pt,outer xsep=2.2pt](x){#2};
\draw[shorten <=3.4pt,shorten >=3.4pt,dashed,#1](x.south west)--(x.south east);
}
}

\newcommand{\DEBUG}{}

\ifdefined\DEBUG%

  \def\rem#1{{\marginpar{\raggedright\scriptsize #1}}}
  \newcommand{\pmr}[1]{\rem{\color{blue}{$\bullet$ #1}}}
  \newcommand{\ppr}[1]{\rem{\color{red}{$\bullet$ #1}}}
 \else%

  \newcommand{\ppr}[1]{}
  \newcommand{\pmr}[1]{}
 \fi

\newcommand{\E}{{\mathbb E}}

\def\rho{\varrho_1}

\theoremstyle{plain}

\newtheorem{lemma}{Lemma}

\newtheorem{proposition}{Proposition}

\DeclareMathOperator{\Pbb}{\mathbb{P}}

\theoremstyle{definition}
\newtheorem{remark}{Remark}

\begin{document}

\title[Neural Network Estimation of Time-Dependent AR($p$) Parameters]{Neural Network-Based Estimation of Time-Dependent Parameters in AR($p$) Processes}

\author[A. Kopeć]{Agnieszka Kopeć}
\address{AGH University of Krakow,
Faculty of Applied Mathematics,
 Al. A.~Mickiewicza 30, 30-059 Krak\'ow, Poland}
\email{agkopec@agh.edu.pl, corresponding author}

\author[P. Przyby\l owicz]{Pawe{\l } Przyby\l owicz}
 \address{AGH University of Krakow,
Faculty of Applied Mathematics,
 Al. A.~Mickiewicza 30, 30-059 Krak\'ow, Poland}
 \email{pprzybyl@agh.edu.pl}

\author[M. Wi\k{a}cek]{Martyna Wi\k{a}cek}
 \address{AGH University of Krakow,
Faculty of Applied Mathematics,
 Al. A.~Mickiewicza 30, 30-059 Krak\'ow, Poland}
\email{martynawiacek@agh.edu.pl}

\begin{abstract}
We investigate a forecasting framework based on a simple discrete-time dynamic model with coefficients varying in time. The parameters of the model are recovered within a deep learning framework, which makes it possible to retain a transparent parametric structure while simultaneously accounting for complex and nonstationary patterns in the observed phenomenon.

Our analysis covers two specifications of the noise process. Besides the standard Gaussian setting, we also consider Laplace-distributed noise, which can offer a more adequate description in the presence of heavier tails and sharper local fluctuations. For both cases, we formulate the predictive scheme of the model and analyze the associated uncertainty quantification, including the construction of prediction intervals.

The results illustrate that a relatively simple model, when combined 
with time-dependent parameter estimation, can serve as a mathematically tractable and practically flexible tool for forecasting complex dynamics under different noise assumptions. The general model is stated for TVAR($p$), while the prediction-interval formulas and the numerical experiments are developed for the TVAR(1) case.
\newline
\newline
\textbf{Key words:} artificial neural networks, time-varying 
parameters, AR(p) process, quasi-likelihood function, 
nonstationary time series, parameter estimation
\newline
\newline
\textbf{MSC 2020:} 62M10, 68T07, 62F10, 60G25
\end{abstract}
\maketitle
\section{Introduction}
Time series models with time-varying parameters constitute a fundamental tool for modeling nonstationary phenomena encountered in economics, finance, energy markets, and environmental applications. Classical autoregressive models assume constant coefficients and stationary noise, which often proves inadequate in the presence of structural changes, evolving volatility, or regime-dependent dynamics. To overcome these limitations, time-varying autoregressive (TVAR) models have been extensively studied in the literature.

A classical and well-established approach to TVAR modeling is based on state-space formulations, in which the autoregressive coefficients are treated as latent stochastic processes, typically evolving according to random walk or autoregressive dynamics. This framework underlies dynamic linear models and Bayesian filtering techniques, such as the Kalman filter and its extensions, and provides a coherent probabilistic interpretation together with recursive estimation procedures; see, for example, \cite{PradoWest,DurbinKoopman}. While these methods are theoretically sound and interpretable, they rely on relatively simple assumptions on the temporal evolution of parameters, which may limit their ability to capture nonlinear or irregular time variation.


An alternative line of research models time-varying parameters as smooth deterministic functions of time. These functions can be estimated using penalized likelihood methods, spline-based techniques, or other smoothness-inducing regularization approaches \cite{KitagawaGersch}. Related frameworks include locally stationary models, which assume that the underlying process evolves gradually over time \cite{Dahlhaus}. Although these methods offer greater flexibility than purely stochastic parameter evolution, they still require the smoothness structure or functional form to be specified a priori.

In parallel, the rapid development of machine learning methods has led to widespread use of neural networks for time series forecasting, including hybrid ARIMA--ANN approaches that combine linear autoregressive structure with nonlinear neural-network components \cite{Zhang2003}. Deep learning models have demonstrated strong empirical performance, particularly in large-scale and highly nonlinear settings. Recent probabilistic forecasting approaches combine neural networks with likelihood-based objectives in order to learn conditional distributions of future observations rather than point predictions \cite{DSM, DeepAR,TranProb}. However, such methods are typically formulated as black-box predictors and do not preserve the explicit autoregressive structure of classical time series models, which limits interpretability and hinders direct statistical inference on model parameters.

More recently, hybrid approaches have emerged that combine stochastic modeling with neural networks, using deep learning as a flexible function approximator for unknown or time-dependent model components. In this paradigm, neural networks are not used to replace the underlying stochastic model, but rather to learn its evolving parameters from data. This idea has been successfully applied to Markov models and stochastic differential equations with time-dependent coefficients, where likelihood-based training allows for statistically principled parameter estimation \cite{sdes_nn}. Such approaches retain the interpretability of classical models while benefiting from the expressive power of neural networks.

A closely related contribution is the DeepTVAR model, where an LSTM network is used to generate time-varying autoregressive coefficient matrices and innovation covariance matrices in a VAR($p$) framework \cite{DeepTVAR}. DeepTVAR also enforces the causality condition on the autoregressive coefficients through the Ansley--Kohn transform and is applied to energy price forecasting. Compared with DeepTVAR, the present work considers a less general, univariate AR($p$) setting, but it differs in several important aspects. First, we use a feedforward neural network rather than a recurrent LSTM architecture. Second, our model explicitly includes a~time-dependent intercept, or trend component, $c(t)$, together with time-varying autoregressive coefficients and a time-varying noise scale. Third, we study both Gaussian and Laplace innovation specifications and derive prediction intervals for the time-varying AR(1) case.

Another related approach is the TVPANN-GARCH model of \cite{TVPANNGARCH}, where artificial neural networks are used to model time-varying parameters in a GARCH-type volatility equation. This work is conceptually close to ours in its use of neural networks as flexible approximators of parameter dynamics, but its statistical object is different: it focuses on conditional volatility modeling, whereas our framework concerns autoregressive dynamics with time-dependent mean and noise parameters.

Motivated by these developments, we focus on the estimation of time-dependent parameters in autoregressive time series models. In contrast to classical state-space TVAR formulations, we model the intercept, autoregressive coefficient, and noise scale as deterministic but unknown functions of time, approximated by a feedforward neural network. Parameter estimation is carried out by minimizing the negative log-likelihood corresponding to the assumed noise distribution, thereby embedding the neural network directly into a likelihood-based inference framework. As a result, the learned network outputs can be interpreted as time-varying autoregressive parameters, which preserves interpretability while allowing the model to capture complex nonstationary behavior.

In addition to the standard Gaussian assumption, we explicitly consider Laplace-distributed noise in order to account for heavier tails and increased robustness to atypical observations. Heavy-tailed noise models are well known to be important in financial and energy time series, where extreme observations occur frequently \cite{KotzLaplace,WeronReview}. By allowing for alternative likelihoods within the same modeling framework, we enable a direct comparison between Gaussian and non-Gaussian specifications and assess their impact on parameter estimation and forecasting performance.

A further contribution of this work is the derivation of recursive formulas for point forecasts and multi-step-ahead prediction intervals in the time-varying AR(1) case, treated as a special case of the general TVAR($p$) framework. Conditional on the estimated parameter trajectories, we provide explicit expressions for forecast means and variances, leading to computationally efficient probabilistic forecasts. Unlike generic neural network forecasting models, our approach preserves the autoregressive structure and allows uncertainty quantification to be carried out in a manner analogous to classical time series analysis.

The proposed methodology is illustrated through numerical experiments on both synthetic and real-world datasets, including electricity spot prices. The experiments are intended as illustrative evidence that the neural-network-based estimation can recover time-varying parameters in controlled settings and produce interpretable forecasts in nonstationary time series. Overall, the paper contributes to the growing literature on interpretable, likelihood-based neural network models for stochastic time series with time-dependent structure.

\section{Model specification and loss function} \label{Section: model and loss}

\subsection{\texorpdfstring{TVAR($p$) model with time-varying coefficients}
{TVAR(p) model with time-varying coefficients}}

We consider a time-varying autoregressive process of order $p$, defined by
\begin{equation}
    y_t = c(t) + \sum_{j=1}^{p} \Phi_j(t)\, y_{t-j} + \varepsilon_t,
    \qquad t \ge p,\quad p \ge 1,
    \label{eq:tvarp_model}
\end{equation}
with initial values
\[
y_0, y_1, \ldots, y_{p-1} \in \mathbb{R}.
\]

The functions
\[
c(\cdot),\qquad \Phi_1(\cdot),\ldots,\Phi_p(\cdot)
\]
are assumed to be unknown and potentially nonlinear. Depending on the assumed distribution of the noise process, the model also involves an additional time-dependent scale parameter, denoted by $\sigma^2(t)$ in the Gaussian case and by $b(t)$ in the Laplace case.

Our goal is to estimate these time-dependent components based on the observed sample
\[
\{\tilde y_0,\tilde y_1,\tilde y_2,\ldots,\tilde y_N\}.
\]

\subsection{\texorpdfstring{Discrete Markov process of order $p \ge 1$}{Discrete Markov process of order p >= 1}}

Let $(Y_t)_{t \in \mathbb{Z}}$ be a discrete Markov process of order $p$. Then, for every $t \ge p$ and $A \in \mathcal{B}(\mathbb{R})$, we have
\[
P(Y_t \in A \mid Y_{t-1}, Y_{t-2}, \ldots, Y_{t-p}, \ldots, Y_0)
=
P(Y_t \in A \mid Y_{t-1}, \ldots, Y_{t-p}).
\]

Assuming that the corresponding conditional densities exist, we obtain
the density of the non-initial part of the trajectory,
\[
f_{(Y_p, \ldots, Y_N)\mid (Y_0,\ldots,Y_{p-1})}
(y_p, \ldots, y_N \mid y_0,\ldots,y_{p-1})
=
\prod_{t=p}^{N}
f_{Y_t \mid (Y_{t-1}, \ldots, Y_{t-p})}
\bigl( y_t \mid y_{t-1}, \ldots, y_{t-p} \bigr).
\]
The variables $Y_0,\ldots,Y_{p-1}$ are treated as fixed initial values in this conditional likelihood.

Model~\eqref{eq:tvarp_model} defines a discrete Markov process of order $p$.Thus, conditional on the block $(y_{t-1},\ldots,y_{t-p})$, the distribution of $y_t$ is fully determined by the time-dependent coefficients and the distribution of the noise term.

\subsection{Gaussian specification and loss function}

Assume that the noise sequence $(\varepsilon_t)$ consists of independent random variables such that
\[
\varepsilon_t \sim \mathcal{N}(0,\sigma^2(t)),
\qquad t \ge p,
\]
where $\sigma^2(t)>0$ for all $t\ge p$.

Then
\[
y_{t+p} \mid y_{t+p-1}, \ldots, y_t
\sim
\mathcal{N}\!\left(
c(t+p) + \sum_{j=1}^{p} \Phi_j(t+p)\, y_{t+p-j},
\ \sigma^2(t+p)
\right),
\]
and the corresponding conditional density is
\begin{equation*}
\begin{split}
    &f_{\, y_{t+p} \mid (y_{t+p-1}, \ldots, y_t)}
    \bigl(x \mid y_{t+p-1}, \ldots, y_t\bigr) \\
    &=
    \frac{1}{\sqrt{2\pi}\,\sigma(t+p)}
    \exp\!\left(
    -
    \frac{
    \Bigl(
    x
    - c(t+p)
    - \sum_{j=1}^{p} \Phi_j(t+p)\, y_{t+p-j}
    \Bigr)^2
    }{
    2\,\sigma^2(t+p)
    }
    \right).
\end{split}
\end{equation*}

Hence, the joint density of $(y_p,\ldots,y_N)$ conditional on $(y_0,\ldots,y_{p-1})$ is given by
\[
f_{(y_p,\ldots,y_N)\mid (y_0,\ldots,y_{p-1})}
(y_p,\ldots,y_N \mid y_0,\ldots,y_{p-1})
=
\prod_{t=p}^{N}
f_{\, y_{t} \mid (y_{t-1}, \ldots, y_{t-p})}
\bigl( y_{t} \mid y_{t-1}, \ldots, y_{t-p} \bigr).
\]
This leads to the negative log-likelihood function
\[
\mathcal{L}_G(c,\Phi,\sigma^2)
=
\sum_{t=p}^{N}
\left[
\frac{1}{2}\ln\!\bigl(2\pi\,\sigma^2(t)\bigr)
+
\frac{
\Bigl(
\tilde y_{t}
-
c(t)
-
\sum_{j=1}^{p} \Phi_j(t)\,\tilde y_{t-j}
\Bigr)^2
}{
2\,\sigma^2(t)
}
\right].
\]

For Gaussian innovations with a fixed or known scale, minimizing this conditional negative log-likelihood reduces to a conditional least-squares criterion, a classical estimation approach for stochastic processes \cite{KlimkoNelson1978}.

We denote the time-varying parameters by
\[
\Theta_G(t)=\bigl[c(t),\,\Phi_1(t),\,\ldots,\,\Phi_p(t),\,\sigma^2(t)\bigr].
\]
Following the approach proposed in \cite{sdes_nn}, we approximate the function $\Theta_G(t)$ using a neural network with weights $w \in \mathbb{R}^M$. The corresponding empirical loss function is obtained by substituting the network outputs into $\mathcal{L}_G$.

\subsection{Laplace specification and loss function}

Assume that the noise sequence $(\varepsilon_t)$ consists of independent random variables such that
\[
\varepsilon_t \sim \mathrm{Laplace}(0,b(t)),
\qquad t \ge p,
\]
where $b(t)>0$ for all $t\ge p$.

Then
\[
y_{t+p} \mid y_{t+p-1}, \ldots, y_t
\sim
\mathrm{Laplace}\!\left(
c(t+p) + \sum_{j=1}^{p} \Phi_j(t+p)\, y_{t+p-j},
\ b(t+p)
\right),
\]
and the corresponding conditional density is
\begin{equation*}
\begin{split}
&f_{\, y_{t+p} \mid (y_{t+p-1}, \ldots, y_t)}
\bigl(x \mid y_{t+p-1}, \ldots, y_t\bigr) \\
&=
\frac{1}{2\,b(t+p)}
\exp\!\left(
-
\frac{
\left|
x
- c(t+p)
- \sum_{j=1}^{p} \Phi_j(t+p)\, y_{t+p-j}
\right|
}{ b(t+p) }
\right).
\end{split}
\end{equation*}

Hence, the joint density of $(y_p,\ldots,y_N)$ conditional on $(y_0,\ldots,y_{p-1})$ is given by
\[
f_{(y_p,\ldots,y_N)\mid (y_0,\ldots,y_{p-1})}
(y_p,\ldots,y_N \mid y_0,\ldots,y_{p-1})
=
\prod_{t=p}^{N}
f_{\, y_{t} \mid (y_{t-1}, \ldots, y_{t-p})}
\bigl( y_{t} \mid y_{t-1}, \ldots, y_{t-p} \bigr).
\]

This leads to the negative log-likelihood function
\[
\mathcal{L}_L(c,\Phi,b)
=
\sum_{t=p}^{N}
\left[
\ln\!\bigl(2\,b(t)\bigr)
+
\frac{
\left|
\tilde y_{t}
-
c(t)
-
\sum_{j=1}^{p} \Phi_j(t)\,\tilde y_{t-j}
\right|
}{
b(t)
}
\right].
\]

We denote the time-varying parameters by
\[
\Theta_L(t)=\bigl[c(t),\,\Phi_1(t),\,\ldots,\,\Phi_p(t),\,b(t)\bigr].
\]
Again, these parameter functions are approximated by a feedforward neural network with weights $w \in \mathbb{R}^M$. The corresponding empirical loss function is obtained by substituting the network outputs into $\mathcal{L}_L$.

\subsection{\texorpdfstring{AR(1) model}
{AR(1) model}}

For $p=1$, model~\eqref{eq:tvarp_model} reduces to the 
process
\begin{equation*}
    y_t = c(t) + \Phi(t) y_{t-1} + \varepsilon_t,
    \qquad t \in \{1,2,3,\ldots\},
    \label{eq:model} 
\end{equation*}
with the initial condition $y_0 \in \mathbb{R}$.

In the Gaussian case, we assume that
\[
\varepsilon_t \sim \mathcal{N}(0,\sigma^2(t)),
\qquad t\ge 1.
\]
Then
\[
y_t \mid y_{t-1}
\sim
\mathcal{N}\bigl(c(t)+\Phi(t)y_{t-1},\,\sigma^2(t)\bigr),
\]
and the corresponding negative log-likelihood function is
\begin{equation}
\begin{split}
    \mathcal{L}(c, \Phi, \sigma^2)
    &=
    \sum_{t=0}^{N-1}
    \left[
        \frac{1}{2}\ln\bigl(2\pi \sigma^2(t+1)\bigr)
        +
        \frac{\bigl(\tilde y_{t+1}-c(t+1)-\Phi(t+1)\tilde y_t\bigr)^2}{2\sigma^2(t+1)}
    \right].
    \label{eq:negloglik}
\end{split}
\end{equation}

We denote the time-varying parameters by
\[
\Theta(t):=\bigl[c(t),\,\Phi(t),\,\sigma^2(t)\bigr],
\]
and approximate this function by a neural network:
\[
\hat{\Theta}(t; w):=\bigl[\hat{c}(t; w),\,\hat{\Phi}(t; w),\,\hat{\sigma}^2(t; w)\bigr].
\]
Substituting these approximations into Equation~\eqref{eq:negloglik} yields the empirical loss function
\begin{equation*}
    \hat{\mathcal{L}}(w)
    =
    \sum_{t=0}^{N-1}
    \left[
        \frac{1}{2}\ln\bigl(2\pi \hat{\sigma}^2(t+1; w)\bigr)
        +
        \frac{\left(\tilde{y}_{t+1} - \hat{c}(t+1; w) - \hat{\Phi}(t+1; w)\tilde{y}_t\right)^2}{2 \hat{\sigma}^2(t+1; w)}
    \right].
\end{equation*}

In the Laplace case, we assume that
\[
\varepsilon_t \sim \mathrm{Laplace}(0,b(t)),
\qquad t\ge 1.
\]
Then
\[
y_t \mid y_{t-1}
\sim
\mathrm{Laplace}\bigl(c(t) + \Phi(t)y_{t-1},\, b(t)\bigr),
\]
and the corresponding negative log-likelihood function is
\[
\mathcal{L}(c,\Phi,b)
=
\sum_{t=0}^{N-1}
\left[
\ln\bigl(2b(t+1)\bigr)
+
\frac{|\tilde y_{t+1} - c(t+1) - \Phi(t+1)\tilde y_t|}{b(t+1)}
\right].
\]

In this case, we denote
\[
\Theta(t) = \bigl[c(t),\,\Phi(t),\,b(t)\bigr],
\]
and approximate this function by a neural network:
\[
\hat{\Theta}(t;w)=\bigl[\hat c(t;w),\,\hat\Phi(t;w),\,\hat b(t;w)\bigr].
\]
Substituting the network outputs into the likelihood yields the empirical loss
\[
\hat{\mathcal{L}}(w)
=
\sum_{t=0}^{N-1}
\left[
\ln\bigl(2\hat b(t+1;w)\bigr)
+
\frac{
|\tilde y_{t+1} - \hat c(t+1;w) - \hat\Phi(t+1;w)\tilde y_t|
}{
\hat b(t+1;w)
}
\right].
\]

\subsection{Normal--Laplace model}

As an additional extension, one may consider the case where the noise term follows the Normal--Laplace distribution
\[
\varepsilon_t \sim \mathrm{NL}\bigl(0,\sigma^2(t),b(t)\bigr).
\]
A random variable $X$ has the Normal--Laplace distribution if
\[
X = W + Z,
\]
where $W$ and $Z$ are independent random variables such that
\[
W \sim \mathcal{N}(\mu,\sigma^2),
\qquad
Z \sim \mathrm{Laplace}(0,b).
\]

Thus, both the Gaussian and Laplace distributions may be viewed as special or limiting cases of this specification. Since the corresponding density is more involved, in the present paper we restrict further analysis to the Gaussian and Laplace cases.

\section{Predictions and prediction intervals in TVAR(1) models}

In this section, we study point predictions and prediction intervals for the TVAR(1) model, treated as a special case of the general framework introduced in Section~2. We assume that the time-dependent parameters have already been estimated based on the observed sample
\[
y_0,y_1,\ldots,y_N,
\]
and we write the corresponding estimates as
\[
\hat c(t), \qquad \hat \Phi(t),
\]
together with the estimated scale parameter, denoted by $\hat\sigma(t)$ in the Gaussian case and by $\hat b(t)$ in the Laplace case. 


The prediction intervals derived in this section are conditional on the information available at the forecasting origin. Since the estimated parameter trajectories are obtained from this information, the intervals account only for the variability induced by future innovations and do not include the additional uncertainty arising from parameter estimation.


\subsection{Prediction intervals for the TVAR(1) model with Gaussian noise}

We consider the model
\[
y_t = c(t) + \Phi(t) y_{t-1} + \sigma(t)\eta_t,
\qquad t \ge 1,
\]
where the sequence $(\eta_t)_{t\ge 1}$ consists of independent and identically distributed random variables such that
\[
\eta_t \sim \mathcal N(0,1).
\]

Given the estimated parameter functions, for $t\ge N+1$ we define the predictive scheme
\[
\hat y_t = \hat c(t) + \hat \Phi(t)\hat y_{t-1} + \hat \sigma(t)\eta_t,
\qquad \hat y_N := y_N.
\]
Moreover, we define a filtration
\[
    \mathcal{G}_0 \coloneqq \{\varnothing, \Omega\},
\]
and for \( t \geq 1 \),
\[
    \mathcal{G}_t \coloneqq \sigma(\eta_1, \ldots, \eta_t),
\]
where \(\sigma(\eta_1, \ldots, \eta_t)\) denotes the \(\sigma\)-algebra generated by the random variables \(\eta_1, \ldots, \eta_t\).
Note that for all \( t \geq N \),
\[
    \sigma(\hat{y}_t) \subset \mathcal{G}_t.
\]
We assume that all estimated parameter functions are 
measurable with respect to the $\mathcal{G}_N$.
For $k\ge 1$, we introduce the notation
\[
\bar y_{N+k} := \mathbb E(\hat y_{N+k}\mid \mathcal G_N).
\]
Since $\eta_{N+k}$ is independent of $\mathcal G_N$ and has zero mean, 
we obtain recursively
\[
\bar y_{N+1} = \hat c(N+1) + \hat \Phi(N+1)y_N,
\]
and, for $k\ge 2$,
\[
\bar y_{N+k}
=
\hat c(N+k) + \hat \Phi(N+k)\bar y_{N+k-1}.
\]

Let
\[
e_{N+k} := \hat y_{N+k} - \bar y_{N+k}
\]
denote the prediction error. Then
\[
e_{N+1} = \hat \sigma(N+1)\eta_{N+1},
\]
and for $k\ge 2$,
\[
e_{N+k}
=
\hat \Phi(N+k)e_{N+k-1} + \hat \sigma(N+k)\eta_{N+k}.
\]

Hence, conditionally on $\mathcal G_N$, the random variable $e_{N+k}$ is Gaussian with zero mean. Its conditional variance can be computed recursively as
\[
s_{N+1}^2 := \operatorname{Var}(e_{N+1}\mid \mathcal G_N)
= \hat \sigma^2(N+1),
\]
and, for $k\ge 2$,
\[
s_{N+k}^2
:=
\operatorname{Var}(e_{N+k}\mid \mathcal G_N)
=
\hat \Phi^2(N+k)s_{N+k-1}^2 + \hat \sigma^2(N+k).
\]

Therefore,
\[
e_{N+k}\mid \mathcal G_N \sim \mathcal N(0,s_{N+k}^2),
\]
and for a confidence level $\alpha\in(0,1)$ we obtain
\[
\mathbb P\!\left(
\left.
\frac{|e_{N+k}|}{s_{N+k}} \le q_\alpha^{(G)}
\,\right|\, \mathcal G_N
\right)=\alpha,
\]
where \(q_\alpha^{(G)}\) denotes the \(\alpha\)-quantile of \(|Z|\) for \(Z\sim\mathcal N(0,1)\), that is, the \(\alpha\)-quantile of the standard half-normal distribution.

Thus, the \(k\)-step-ahead prediction interval at level \(\alpha\) is
\[
\bar P_{N+k}^{(G)}(\alpha)
:=
\left[
\bar y_{N+k} - q_\alpha^{(G)} s_{N+k},
\;
\bar y_{N+k} + q_\alpha^{(G)} s_{N+k}
\right].
\]

\subsection{Prediction intervals for the TVAR(1) model with Laplace noise}

We recall a~basic property of the Laplace distribution that will be used throughout this subsection.
\begin{remark}[Scaling property of the Laplace distribution]
If $X \sim \mathrm{Laplace}(0,1)$, then for any $\beta \in \mathbb{R}$ and
$\alpha > 0$,
\[
\alpha X + \beta \sim \mathrm{Laplace}(\beta, \alpha).
\]
\end{remark}

We now consider the model
\[
y_t = c(t) + \Phi(t)y_{t-1} + b(t)\xi_t,
\qquad t \ge 1,
\]
where the sequence $(\xi_t)_{t\ge 1}$ consists of independent and identically distributed random variables such that
\[
\xi_t \sim \mathrm{Laplace}(0,1).
\]
Given the estimated parameter functions, for $t\ge N+1$ we define
\[
\hat y_t = \hat c(t) + \hat \Phi(t)\hat y_{t-1} + \hat b(t)\xi_t,
\qquad \hat y_N := y_N.
\]
As before, we define a filtration
\[
    \mathcal{F}_0 \coloneqq \{\varnothing, \Omega\},
\]
and for \( t \geq 1 \),
\[
    \mathcal{F}_t \coloneqq \sigma(\xi_1, \ldots, \xi_t).
\]
Then for all \( t \geq N \),
\[
    \sigma(\hat{y}_t) \subset \mathcal{F}_t.
\]
As in the Gaussian case, we assume that all estimated parameter functions 
are measurable with respect to the $\mathcal{F}_N$.
We introduce the notation
\[
\bar y_{N+k} := \mathbb E(\hat y_{N+k}\mid \mathcal F_N),
\qquad k\ge 1.
\]
Since $\xi_{N+k}$ is independent of $\mathcal F_N$ and has zero mean,
we obtain
\[
\bar y_{N+1} = \hat c(N+1) + \hat \Phi(N+1)y_N,
\]
and, for $k\ge 2$,
\[
\bar y_{N+k}
=
\hat c(N+k) + \hat \Phi(N+k)\bar y_{N+k-1}.
\]

\subsection*{One-step-ahead prediction}

Let
\[
e_{N+1}:=\hat y_{N+1}-\bar y_{N+1}.
\]
Then
\[
e_{N+1}=\hat b(N+1)\xi_{N+1},
\]
and therefore
\[
e_{N+1}\mid \mathcal F_N \sim \mathrm{Laplace}(0,\hat b(N+1)).
\]

We use the following elementary fact.

\begin{lemma}\label{lem:laplace_abs}
Let \(X\sim \mathrm{Laplace}(0,s)\) with \(s>0\). Then, for every \(q\ge 0\),
\[
\mathbb P(|X|\le q)=1-e^{-q/s}.
\]
\end{lemma}

\begin{proof}
Let \(F\) denote the cumulative distribution function of \(X\). Then
\[
\mathbb P(|X|\le q)=\mathbb P(-q\le X\le q)=F(q)-F(-q).
\]
For \(X\sim \mathrm{Laplace}(0,s)\),
\[
F(x)=
\begin{cases}
\frac12 e^{x/s}, & x\le 0,\\[0.5ex]
1-\frac12 e^{-x/s}, & x>0.
\end{cases}
\]
Hence,
\[
\mathbb P(|X|\le q)
=
\left(1-\frac12 e^{-q/s}\right)-\frac12 e^{-q/s}
=
1-e^{-q/s}.
\]
\end{proof}

Applying Lemma~\ref{lem:laplace_abs}, we obtain
\[
\mathbb P\!\left(
\left.
\frac{|e_{N+1}|}{\hat b(N+1)} \le q_\alpha^{(L)}
\,\right|\, \mathcal F_N
\right)=\alpha
\]
if and only if
\[
1-e^{-q_\alpha^{(L)}}=\alpha.
\]
Thus
\[
q_\alpha^{(L)}=-\ln(1-\alpha)=\ln\!\left(\frac{1}{1-\alpha}\right),
\]
and the one-step-ahead prediction interval at level \(\alpha\) is
\[
\bar P_{N+1}^{(L)}(\alpha)
:=
\left[
\bar y_{N+1}
-
\hat b(N+1)\ln\!\left(\frac{1}{1-\alpha}\right),
\;
\bar y_{N+1}
+
\hat b(N+1)\ln\!\left(\frac{1}{1-\alpha}\right)
\right].
\]

\subsection*{Two-step-ahead prediction}

We now consider the two-step-ahead conditional expectation
\[
\bar y_{N+2}:=\mathbb E(\hat y_{N+2}\mid \mathcal F_N).
\]
Using the model structure, we obtain
\[
\bar y_{N+2}
=
\hat c(N+2)+\hat \Phi(N+2)\bar y_{N+1}.
\]

The corresponding prediction error is
\[
e_{N+2}
:=
\hat y_{N+2}-\bar y_{N+2}
=
\hat \Phi(N+2)e_{N+1}+\hat b(N+2)\xi_{N+2}.
\]
Since
\[
e_{N+1}=\hat b(N+1)\xi_{N+1},
\]
it follows that
\[
e_{N+2}
=
\hat \Phi(N+2)\hat b(N+1)\xi_{N+1}
+
\hat b(N+2)\xi_{N+2}.
\]

Set
\[
a:=|\hat \Phi(N+2)|\,\hat b(N+1),
\qquad
c:=\hat b(N+2).
\]
Then, conditionally on \(\mathcal F_N\),
\[
e_{N+2}\overset{d}=X_1+X_2,
\]
where
\[
X_1\sim \mathrm{Laplace}(0,a),
\qquad
X_2\sim \mathrm{Laplace}(0,c),
\]
and \(X_1\) and \(X_2\) are independent.

If \(a\neq c\), the density of \(e_{N+2}\) is given by
\[
f_{e_{N+2}}(x)
=
\frac{a\,e^{-|x|/a}-c\,e^{-|x|/c}}{2(a^2-c^2)},
\qquad x\in\mathbb R.
\]
Consequently, for \(q\ge 0\),
\[
\mathbb P(|e_{N+2}|\le q\mid \mathcal F_N)
=
1-
\frac{a^2e^{-q/a}-c^2e^{-q/c}}{a^2-c^2}.
\]
Therefore, the radius \(q_\alpha^{(2)}\) of the two-step-ahead 
prediction interval is determined by the equation
\[
a^2e^{-q_\alpha^{(2)}/a}
-
c^2e^{-q_\alpha^{(2)}/c}
=
(1-\alpha)(a^2-c^2).
\]

If \(a=c\), then
\[
f_{e_{N+2}}(x)
=
\frac{1}{4a}e^{-|x|/a}\left(1+\frac{|x|}{a}\right),
\qquad x\in\mathbb R,
\]
and, for \(q\ge 0\),
\[
\mathbb P(|e_{N+2}|\le q\mid \mathcal F_N)
=
1-e^{-q/a}\left(1+\frac{q}{2a}\right).
\]
Hence, in this case \(q_\alpha^{(2)}\) is determined by
\[
e^{-q_\alpha^{(2)}/a}\left(1+\frac{q_\alpha^{(2)}}{2a}\right)
=
1-\alpha.
\]

The detailed derivation of the density of \(e_{N+2}\), together with the corresponding distribution function and the equation determining the quantile \(q_\alpha^{(2)}\), is given in Appendix~A.

Thus, the two-step-ahead prediction interval at level \(\alpha\) is
\[
\bar P_{N+2}^{(L)}(\alpha)
:=
\left[
\bar y_{N+2}-q_\alpha^{(2)},
\;
\bar y_{N+2}+q_\alpha^{(2)}
\right].
\]

In contrast to the Gaussian case, the quantity \(q_\alpha^{(2)}\) is not given by a universal quantile depending only on \(\alpha\). It depends on the estimated parameters through \(a=|\hat \Phi(N+2)|\hat b(N+1)\) and \(c=\hat b(N+2)\), and is obtained as the unique solution of the corresponding scalar equation above.

\subsection*{Multi-step-ahead prediction}


Finally, we define the $k$-step prediction error 
\begin{equation*}
  e_{N+j} := \widehat y_{N+j}-\overline y_{N+j},\qquad j\in\{0, 1, \dots, k\}.
  \label{eq:error_def}
\end{equation*}
We recall the plug-in predictive path defined by
\begin{equation}
\begin{aligned}
  \widehat y_N &:= y_N,\\
  \widehat y_{N+j} &:= \widehat c(N+j)
  + \widehat\Phi(N+j)\,\widehat y_{N+j-1}
  + \widehat b(N+j)\,\xi_{N+j},
  \qquad j\in\{1, \dots, k\}.
\end{aligned}
\label{eq:plugin_path}
\end{equation}
and the associated conditional mean (“noise-free”) prediction given by
\begin{equation}
\begin{aligned}
  \overline y_N &:= y_N,\\
  \overline y_{N+j} &:= \E\big[\widehat y_{N+j}\mid\mathcal{F}_N\big]
  = \widehat c(N+j) + \widehat\Phi(N+j)\,\overline y_{N+j-1},
  \qquad j\in\{1, \dots, k\}.
\end{aligned}
\label{eq:mean_path}
\end{equation}

\begin{proposition}[Exact conditional law of the $k$-step prediction error under Laplace innovations]
\label{prop:laplace_kstep}
Assume the foregoing set-up. Then, conditional on $\mathcal{F}_N$:
\begin{enumerate}
\item
\label{item:1}
The error process satisfies the recursion
\begin{equation}
  e_{N} = 0,\qquad e_{N+j}=\widehat\Phi(N+j)\,e_{N+j-1}+\widehat b(N+j)\,\xi_{N+j},\qquad j\in\{1, \dots, k\},
  \label{eq:error_recursion}
\end{equation}
and admits the explicit expansion
\begin{equation}
  e_{N+k}=\sum_{j=1}^k a_{j,k}\,\xi_{N+j},
  \qquad
  a_{j,k}:=\widehat b(N+j)\prod_{m=j+1}^k \widehat\Phi(N+m),
  \label{eq:error_expansion}
\end{equation}
with the convention that an empty product equals~$1$.

\item
\label{item:2}
The conditional characteristic function of $e_{N+k}$ is
\begin{equation}
  \varphi_{k}(t)
  :=\E\big[e^{it e_{N+k}}\mid\mathcal{F}_N\big]
  =\prod_{j=1}^k \frac{1}{1+a_{j,k}^2 t^2},\qquad t\in\mathbb{R}.
  \label{eq:cf_product}
\end{equation}

\item 
\label{item:3}
Let $s_{j,k}:=|a_{j,k}|$ for $j\in\{1, \dots, k\}$ and assume that $s_{1,k},\dots,s_{k,k}$ are pairwise distinct and strictly positive.
Define
\begin{equation}
  A_{j,k}:=\prod_{\ell\in\{1,\dots,k\}\setminus\{j\}}\frac{s_{j,k}^2}{s_{j,k}^2-s_{\ell,k}^2}.
  \label{eq:Ajk}
\end{equation}
Then $e_{N+k}$ has a Lebesgue density given by the finite sum
\begin{equation}
  f_{k}(x)
  :=\frac{\mathrm{d}}{\mathrm{d}x}\Pbb(e_{N+k}\le x\mid\mathcal{F}_N)
  =\frac12\sum_{j=1}^k \frac{A_{j,k}}{s_{j,k}}\exp\Big(-\frac{|x|}{s_{j,k}}\Big),\qquad x\in\mathbb{R}.
  \label{eq:pdf_sumexp}
\end{equation}
In particular, the conditional distribution is symmetric about~$0$ and for $x\ge 0$,
\begin{equation}
  F_k(x):=\Pbb(e_{N+k}\le x\mid\mathcal{F}_N)
  =1-\frac12\sum_{j=1}^k A_{j,k}\exp\Big(-\frac{x}{s_{j,k}}\Big).
  \label{eq:cdf_xge0}
\end{equation}

\item 
\label{item:4}
For any confidence level $\alpha\in(0,1)$ there exists a unique $q_{k,\alpha}>0$ such that
\begin{equation}
  \Pbb\big(|e_{N+k}|\le q_{k,\alpha}\mid\mathcal{F}_N\big)=\alpha.
  \label{eq:q_def}
\end{equation}
Under the distinct-scales assumption from item~\ref{item:3}, $q_{k,\alpha}$ is the unique solution of
\begin{equation}
  \sum_{j=1}^k A_{j,k}\exp\Big(-\frac{q_{k,\alpha}}{s_{j,k}}\Big)=1-\alpha.
  \label{eq:q_equation}
\end{equation}
Consequently, a central $\alpha$-level prediction interval for $y_{N+k}$ is
\begin{equation}
  \big[\overline y_{N+k}-q_{k,\alpha},\ \overline y_{N+k}+q_{k,\alpha}\big].
  \label{eq:central_PI}
\end{equation}
\end{enumerate}
\end{proposition}
\begin{proof}
We prove the four items in turn. 

\medskip
\noindent Ad \ref{item:1}. From \eqref{eq:plugin_path} and \eqref{eq:mean_path}, we have, for $j\ge 1$,
\begin{align*}
  e_{N+j}
  &= \widehat y_{N+j}-\overline y_{N+j}\\
  &= \Big(\widehat c(N+j) + \widehat\Phi(N+j)\widehat y_{N+j-1} + \widehat b(N+j)\xi_{N+j}\Big)
     -\Big(\widehat c(N+j)+\widehat\Phi(N+j)\overline y_{N+j-1}\Big)\\
  &= \widehat\Phi(N+j)\,\big(\widehat y_{N+j-1}-\overline y_{N+j-1}\big)+\widehat b(N+j)\xi_{N+j}
   = \widehat\Phi(N+j)\,e_{N+j-1}+\widehat b(N+j)\xi_{N+j}.
\end{align*}
Since $e_N=\widehat y_N-\overline y_N=y_N-y_N=0$, this proves \eqref{eq:error_recursion}.

We now show \eqref{eq:error_expansion} by induction on~$k$.
For $k=1$, \eqref{eq:error_recursion} gives $e_{N+1}=\widehat b(N+1)\xi_{N+1}$, which matches \eqref{eq:error_expansion} with $a_{1,1}=\widehat b(N+1)$.
Assume \eqref{eq:error_expansion} holds for a given $k\ge 1$.
Then, by \eqref{eq:error_recursion},
\begin{align*}
  e_{N+(k+1)}
  &= \widehat\Phi(N+k+1)e_{N+k}+\widehat b(N+k+1)\xi_{N+k+1}\\
  &= \widehat\Phi(N+k+1)\sum_{j=1}^{k} a_{j,k}\xi_{N+j}+\widehat b(N+k+1)\xi_{N+k+1}\\
  &= \sum_{j=1}^{k}\Big(a_{j,k}\widehat\Phi(N+k+1)\Big)\xi_{N+j}+\widehat b(N+k+1)\xi_{N+k+1}.
\end{align*}
Comparing with \eqref{eq:error_expansion}, we see that $a_{j,k+1}=a_{j,k}\widehat\Phi(N+k+1)$ for $j\le k$ and $a_{k+1,k+1}=\widehat b(N+k+1)$.
This is exactly the definition
\(a_{j,k+1}=\widehat b(N+j)\prod_{m=j+1}^{k+1}\widehat\Phi(N+m)\).
Thus \eqref{eq:error_expansion} holds for $k+1$, completing the induction.

\medskip
\noindent Ad \ref{item:2}. By \eqref{eq:error_expansion} and the conditional independence of $\xi_{N+1},\dots,\xi_{N+k}$ from $\mathcal{F}_N$ (and from each other),
\begin{align*}
  \varphi_k(t)
  &=\E\Big[\exp\Big(it\sum_{j=1}^k a_{j,k}\xi_{N+j}\Big)\,\Big|\,\mathcal{F}_N\Big]
   =\E\Big[\prod_{j=1}^k \exp\big(it a_{j,k}\xi_{N+j}\big)\,\Big|\,\mathcal{F}_N\Big]\\
  &=\prod_{j=1}^k \E\big[\exp\big(it a_{j,k}\xi_{N+j}\big)\big]
   =\prod_{j=1}^k \frac{1}{1+a_{j,k}^2 t^2},
\end{align*}
where in the last step we used Lemma~\ref{lem:cf_laplace} with $s=|a_{j,k}|$.
This proves~\eqref{eq:cf_product}.

\medskip
\noindent Ad \ref{item:3}. Set $s_{j,k}:=|a_{j,k}|$ and assume $s_{1,k},\dots,s_{k,k}$ are pairwise distinct and positive.
Introduce the auxiliary variable $u:=t^2\ge 0$.
Then \eqref{eq:cf_product} becomes
\begin{equation*}
  \varphi_k(t)=\prod_{j=1}^k \frac{1}{1+s_{j,k}^2 t^2}
  =\prod_{j=1}^k \frac{1}{1+s_{j,k}^2 u}.
  \label{eq:cf_as_R}
\end{equation*}
By Lemma~\ref{lem:partial_fraction} (applied with $s_j=s_{j,k}$), there exist constants $A_{j,k}$ given by \eqref{eq:Ajk} such that
\begin{equation}
  \varphi_k(t)=\sum_{j=1}^k \frac{A_{j,k}}{1+s_{j,k}^2 t^2}.
  \label{eq:cf_pf}
\end{equation}
Since $\varphi_k$ is integrable (indeed, $\varphi_k(t)=O(t^{-2k})$ as $|t|\to\infty$), the conditional density exists and can be obtained by Fourier inversion:
\begin{equation}
  f_k(x)=\frac{1}{2\pi}\int_{-\infty}^{\infty} e^{-itx}\,\varphi_k(t)\,\mathrm{d}t.
  \label{eq:fourier_inversion}
\end{equation}
Insert \eqref{eq:cf_pf} into \eqref{eq:fourier_inversion} and use that the sum is finite to exchange sum and integral:
\begin{align*}
  f_k(x)
  &=\sum_{j=1}^k A_{j,k}\,\frac{1}{2\pi}\int_{-\infty}^{\infty}\frac{e^{-itx}}{1+s_{j,k}^2 t^2}\,\mathrm{d}t
   =\sum_{j=1}^k A_{j,k}\,g_{s_{j,k}}(x),
\end{align*}
where $g_s$ is as in Lemma~\ref{lem:inv_fourier}.
Applying Lemma~\ref{lem:inv_fourier} gives exactly \eqref{eq:pdf_sumexp}.

To obtain the CDF for $x\ge 0$, integrate \eqref{eq:pdf_sumexp} from $-\infty$ to $x$.
Using symmetry of each term $e^{-|x|/s}$, it is convenient to first compute the survival function for $x\ge 0$:
\begin{align*}
  \Pbb(e_{N+k}>x\mid\mathcal{F}_N)
  &=\int_{x}^{\infty} f_k(u)\,\mathrm{d}u
   =\frac12\sum_{j=1}^k \frac{A_{j,k}}{s_{j,k}}\int_{x}^{\infty} e^{-u/s_{j,k}}\,\mathrm{d}u\\
  &=\frac12\sum_{j=1}^k \frac{A_{j,k}}{s_{j,k}}\,\Big[-s_{j,k}e^{-u/s_{j,k}}\Big]_{u=x}^{u=\infty}
   =\frac12\sum_{j=1}^k A_{j,k} e^{-x/s_{j,k}}.
\end{align*}
Therefore, for $x\ge 0$,
\(F_k(x)=1-\Pbb(e_{N+k}>x\mid\mathcal{F}_N)\)
which yields \eqref{eq:cdf_xge0}.
Symmetry follows from~\eqref{eq:error_expansion} since each $\xi_{N+j}$ is symmetric about $0$ and the coefficients $a_{j,k}$ are deterministic given~$\mathcal{F}_N$. 
Alternatively, it is immediate from \eqref{eq:pdf_sumexp},
since \(f_k\) is an even function.

\medskip
\noindent Ad \ref{item:4}. Since each scaled Laplace density is strictly positive everywhere and continuous, the density of a sum of independent scaled Laplace variables is the convolution of strictly positive continuous functions and is therefore itself strictly positive and continuous on $\mathbb{R}$.
Hence the conditional CDF of $e_{N+k}$ is continuous and strictly increasing.
Consequently, for any $\alpha\in(0,1)$ there exists a unique $q_{k,\alpha}>0$ satisfying \eqref{eq:q_def}.

Under the distinct-scales assumption, by symmetry
\begin{align*}
  \Pbb(|e_{N+k}|\le q\mid\mathcal{F}_N)
  &=\Pbb(-q\le e_{N+k}\le q\mid\mathcal{F}_N)
   =F_k(q)-F_k(-q)
   =2F_k(q)-1,
\end{align*}
for all $q\ge 0$.
Thus \eqref{eq:q_def} is equivalent to $F_k(q)=(1+\alpha)/2$.
Using \eqref{eq:cdf_xge0} with $x=q\ge 0$, we obtain
\begin{align*}
  \frac{1+\alpha}{2}
  =F_k(q)
  =1-\frac12\sum_{j=1}^k A_{j,k}e^{-q/s_{j,k}},
\end{align*}
which rearranges to \eqref{eq:q_equation}.
The prediction interval \eqref{eq:central_PI} then follows directly from the definition
$e_{N+k} = \widehat{y}_{N+k}-\overline y_{N+k}$.
\end{proof}

\section{Numerical experiments}

In this section, we present numerical experiments on both synthetic and real data. Our aim is twofold: to verify whether the proposed neural-network-based procedure can recover time-dependent model parameters in a controlled setting and to illustrate its forecasting behavior under different noise assumptions. In the synthetic case, the true parameter functions are known and can therefore be directly compared with their estimates. In the real-data case, we analyze the recovered parameter trajectories and report one-step-ahead and two-step-ahead prediction errors obtained under Gaussian and Laplace noise. The temporal ordering of the observations is preserved throughout all experiments, and no random shuffling is used. The real-data forecasts are intended as an illustrative comparison for the considered splits, rather than as a comprehensive forecasting benchmark. The notebooks used to generate the final experimental figures are collected in the folder \texttt{TVAR\_NOTEBOOKS\_FINAL}
and are available at \url{https://github.com/agkopec/TVAR_NOTEBOOKS_FINAL}.


\subsection{Numerical experiments on synthetic data}

We first consider synthetic data generated from the TVAR(1) model under two noise specifications: Gaussian and Laplace. Since the true parameter functions are known in this setting, these experiments provide a direct benchmark for assessing the accuracy of the proposed neural-network-based estimation procedure. In  both synthetic experiments we use $N=99$ transitions, i.e. observations $y_0,\ldots,y_{99}$, with $y_0=7$ and random seed $42$. For $t=0,\ldots,99$, the deterministic parameter functions are
\[
c(t)=\sin\!\left(\frac{t}{6\pi}\right)+7,
\]
\[
\phi(t)=-0.3\left(0.5\cos\!\left(\frac{t}{4\pi}+30\right)+0.5\right)+0.2,
\]
and
\[
b(t)=\left(0.5\left|0.5+\sin\!\left(\frac{t}{6\pi}\right)\right|+0.5\right)^2.
\]
For the Laplace experiment the innovations satisfy $\varepsilon_t\sim\mathrm{Laplace}(0,b(t))$, whereas for the Gaussian experiment we use $\varepsilon_t\sim\mathcal N(0,\sigma^2(t))$ with
$$\sigma^2(t) = b(t).$$
The same functions $c(t)$ and $\phi(t)$ are used in both noise specifications.

For the Laplace noise specification,
the final architecture and the training hyperparameters were selected by a grid-search procedure with multistart initialization. The search covered the hidden-layer widths and depth, activation function, learning rate, smoothness penalty, and the optional constraint imposed on the autoregressive coefficient ($\lvert \phi \rvert < 0.99$). The final model was selected by minimizing the average raw mean squared reconstruction error over the recovered parameter trajectories. In the final Laplace run, grid search selected a three-layer GELU network with hidden widths $(16,16,16)$.
By contrast, for the Gaussian noise specification, the final neural network architecture and training hyperparameters were tuned manually, as grid search failed to produce a satisfactory estimate of the parameter $\sigma^2$. The final Gaussian run therefore uses a manually selected three-layer GELU network with hidden widths $(16,32,16)$. 

Figure~\ref{fig:synth-data} presents representative synthetic trajectories generated under Gaussian and Laplace noise. Figure~\ref{fig:synth-params} compares the true and estimated parameter functions. In the Gaussian case, we report the recovery of $c(t)$, $\phi(t)$, and $\sigma^2(t)$, while in the Laplace case, we report the recovery of $c(t)$, $\phi(t)$, and $b(t)$.

The neural network architectures 
are summarized in Table~\ref{tab:synth-arch}, and the corresponding reconstruction errors are reported in Table~\ref{tab:synth-metrics}. Two hyperparameters are omitted from Table~\ref{tab:synth-arch}, as the same values were used in both cases ($\phi$ constrained: no, smoothness: $0$). In both cases, the selected model is relatively small, which indicates that accurate recovery of the time-dependent coefficients does not require a~highly overparameterized network in the present synthetic setting.

The numerical results show that the proposed method reconstructs the time-varying autoregressive coefficient very accurately under both noise assumptions. The largest discrepancies are observed for the intercept and, in the Gaussian case, for the variance trajectory. Overall, the recovered parameter curves remain close to the true ones, which confirms that the likelihood-based neural-network approach is capable of learning the time-dependent structure of the model in a stable and interpretable way.

\begin{figure}[!htbp]
    \centering
    \subfloat[Gaussian noise]{
        \includegraphics[width=0.48\linewidth]{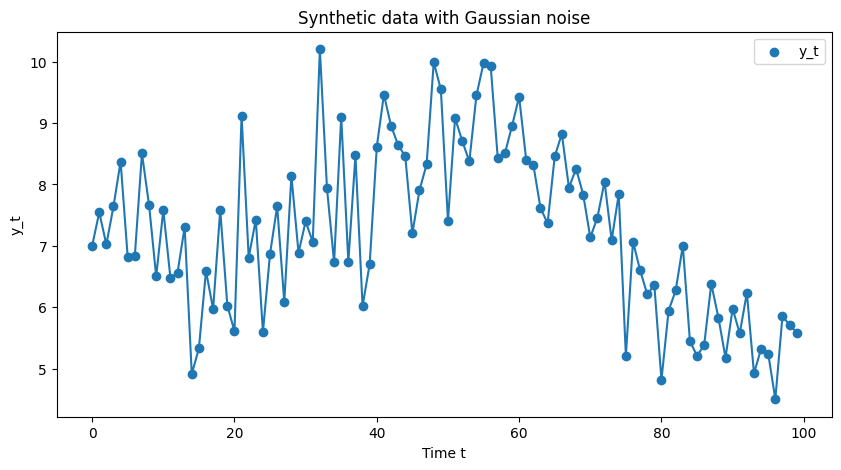}
    }
    \hfill
    \subfloat[Laplace noise]{
        \includegraphics[width=0.48\linewidth]{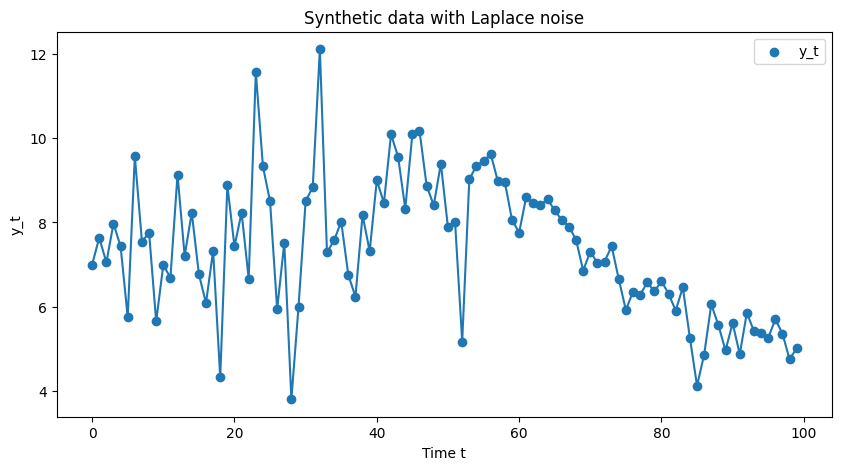}
    }
    \caption{Synthetic trajectories generated from the TVAR(1) model under Gaussian and Laplace noise.}
    \label{fig:synth-data}
\end{figure}

\begin{figure}[tbp]
    \centering
    \subfloat[Gaussian noise -- $c(t)$]{
        \includegraphics[width=0.32\linewidth]{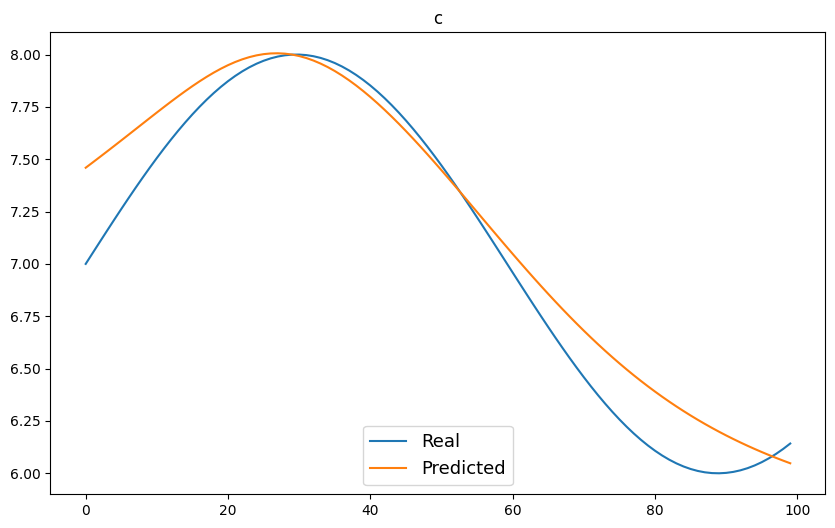}
    }
    \subfloat[Gaussian noise -- $\phi(t)$]{
        \includegraphics[width=0.32\linewidth]{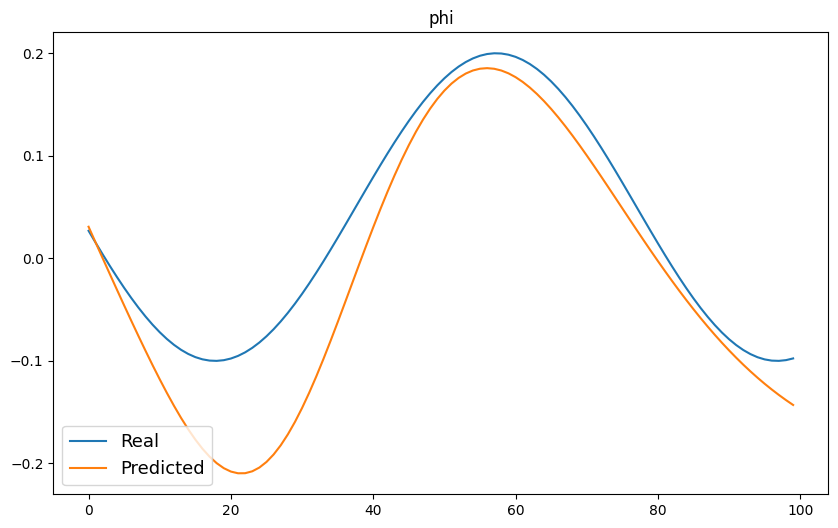}
    }
    \subfloat[Gaussian noise -- $\sigma^2(t)$]{
        \includegraphics[width=0.32\linewidth]{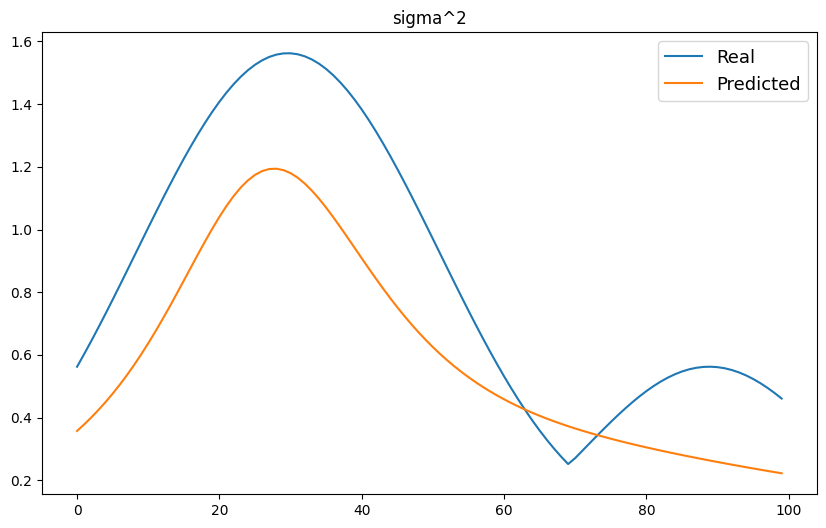}
    }

    \subfloat[Laplace noise -- $c(t)$]{
        \includegraphics[width=0.32\linewidth]{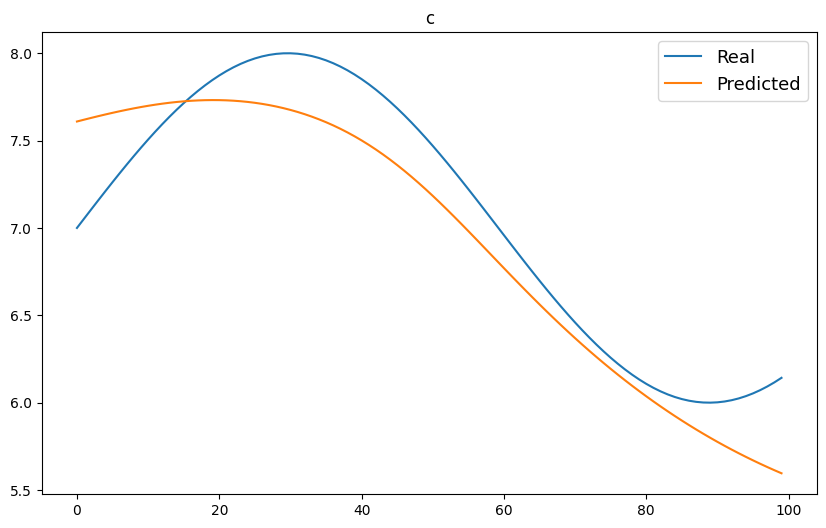}
    }
    \subfloat[Laplace noise -- $\phi(t)$]{
        \includegraphics[width=0.32\linewidth]{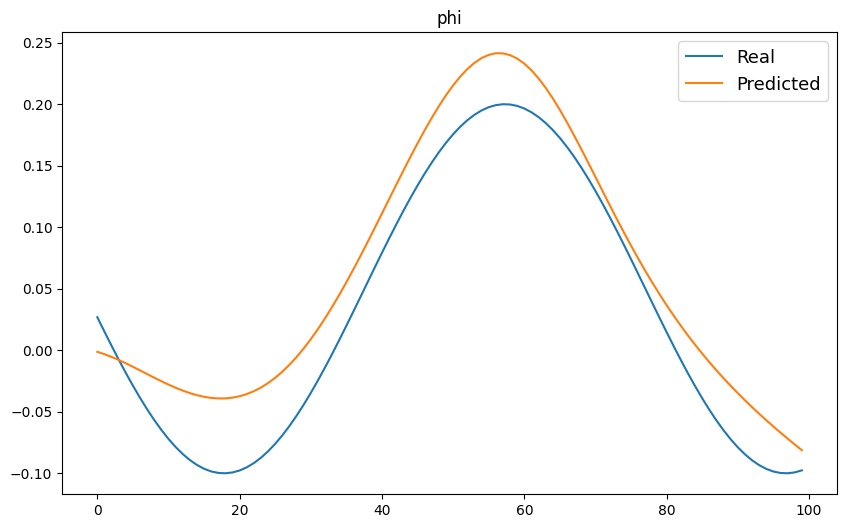}
    }
    \subfloat[Laplace noise -- $b(t)$]{
        \includegraphics[width=0.32\linewidth]{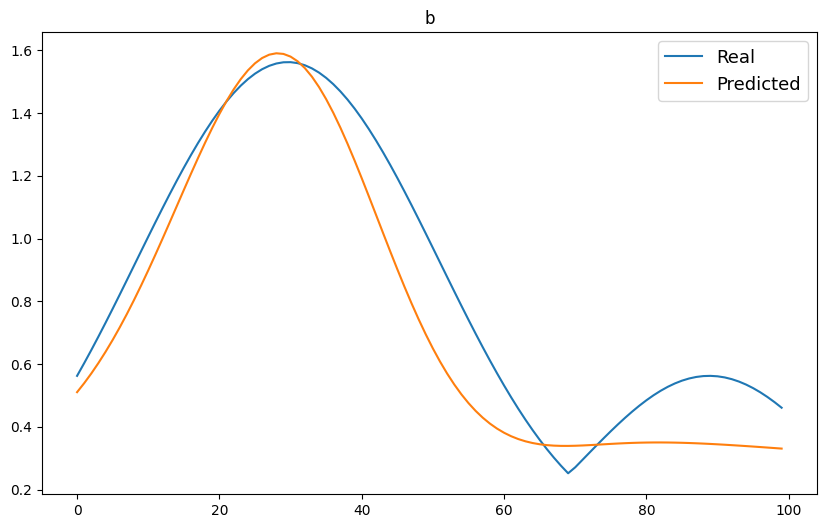}
    }
    \caption{True and estimated time-dependent parameters in the synthetic experiments. Panels (A)--(C) correspond to the Gaussian case, whereas panels (D)--(F) correspond to the Laplace case.}
    \label{fig:synth-params}
\end{figure}



\begin{table}[htbp]
\centering
\begin{tabular}{lcccccc}
\hline
Noise & Hidden layers & Activation & Optimizer & Learning rate & 
Batch size & Epochs \\
\hline
Gaussian & $(16,32,16)$ & GELU & Adam & $10^{-3}$ & $16$ & 2500 \\
Laplace  & $(16,16,16)$ & GELU & AdamW & $3\cdot 10^{-3}$ & $100$ (full) & 2500 \\
\hline
\end{tabular}
\caption{Architectures and training hyperparameters used in the synthetic experiments.}
\label{tab:synth-arch}
\end{table}


\begin{table}[htbp]
\centering
\begin{tabular}{lcccc}
\hline
Noise & $\mathrm{MSE}(c)$ & $\mathrm{MSE}(\phi)$ & $\mathrm{MSE}(\text{scale})$ & $\mathrm{MSE}_{\mathrm{mean}}$ \\
\hline
Gaussian & $0.0320$ & $0.00315$ & $0.0954$ & $0.0435$ \\
Laplace  & $0.0741$ & $0.00142$ & $0.0252$ & $0.0336$ \\
\hline
\end{tabular}
\caption{Parameter reconstruction errors in the synthetic experiments. In the Gaussian case, the scale parameter is $\sigma^2(t)$, while in the Laplace case it is $b(t)$. The quantity $\mathrm{MSE}_{\mathrm{mean}}$ denotes the arithmetic mean of the three parameter-wise MSE values.}
\label{tab:synth-metrics}
\end{table}


Table~\ref{tab:synth-arch} presents the architectures used for both noise specifications. In both cases, relatively small feedforward networks with three hidden layers and GELU activation were employed. The main differences lie in the hidden-layer widths, the optimizer, the batch size, and the method of selecting the final architecture: grid search for Laplace and manual tuning for Gaussian.

Table~\ref{tab:synth-metrics} confirms that the proposed method reconstructs the autoregressive coefficient very accurately in both cases, with $\mathrm{MSE}(\phi)$ of order $10^{-3}$. In the Gaussian case, the dominant contribution to the total reconstruction error comes from the variance trajectory $\sigma^2(t)$, whereas in the Laplace case the scale parameter $b(t)$ is recovered much more accurately. As a consequence, the overall average reconstruction error is smaller in the Laplace experiment.
\subsection{Numerical experiments on real data}

We next consider real data given by electricity spot prices; see Figure~\ref{fig:used-dataset}. The data are taken from the hourly electricity spot-price dataset for Denmark and neighboring countries available at \url{https://www.kaggle.com/datasets/arashnic/electricity-spot-price}. To obtain a univariate daily time series, we fix one price series and retain one observation per day at the same hour throughout the experiment. The resulting observations are ordered chronologically, and the same preprocessing rule is used for all four fitted models. In contrast to the synthetic setting, the true parameter functions are not known; therefore, the numerical study focuses on the interpretability of the recovered parameter trajectories, the quality of trajectories generated by the fitted models, and the resulting illustrative forecasting performance for the selected train-test splits.

We report results for two training-set sizes: 81 and 995 observations. For each sample size, we estimate the time-dependent intercept \(c(t)\), autoregressive coefficient \(\phi(t)\), and the corresponding scale parameter, namely \(\sigma^2(t)\) in the Gaussian case and \(b(t)\) in the Laplace case. 
The fitted networks estimate $\Theta$ over a wider horizon; however, only the estimates $\widehat{\Theta}(N+1)$ and $\widehat{\Theta}(N+2)$ are used for the two-step-ahead recursive forecasting procedure.
In order to make the comparison between the two noise specifications transparent, the network architectures were chosen manually and fixed within each sample-size regime based on preliminary experiments. The resulting architectures are summarized in Table~\ref{tab:real-arch}.

Figure~\ref{fig:real-data-windows} shows the two data windows used in the real-data experiments. Figures~\ref{fig:real-81-params} and~\ref{fig:real-995-params} display the estimated time-dependent parameter trajectories for the two training-set sizes. In both cases, the estimated coefficients vary substantially over time, which supports the use of a nonstationary autoregressive specification.

\begin{figure}[H]
    \centering
    \includegraphics[width=0.75\linewidth]{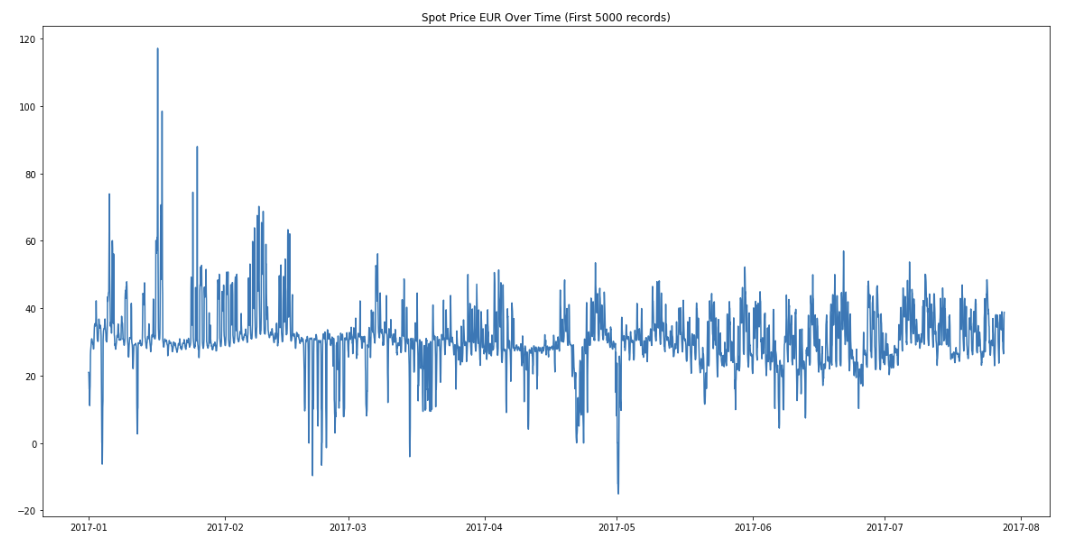}
    \caption{First $5000$ records from the dataset used for numerical experiments on real data: 
    energy spot prices in Denmark and neighboring countries (source: \url{https://www.kaggle.com/datasets/arashnic/electricity-spot-price}).}
    \label{fig:used-dataset}
\end{figure}

\begin{figure}[H]
    \centering
    \includegraphics[width=0.75\linewidth]{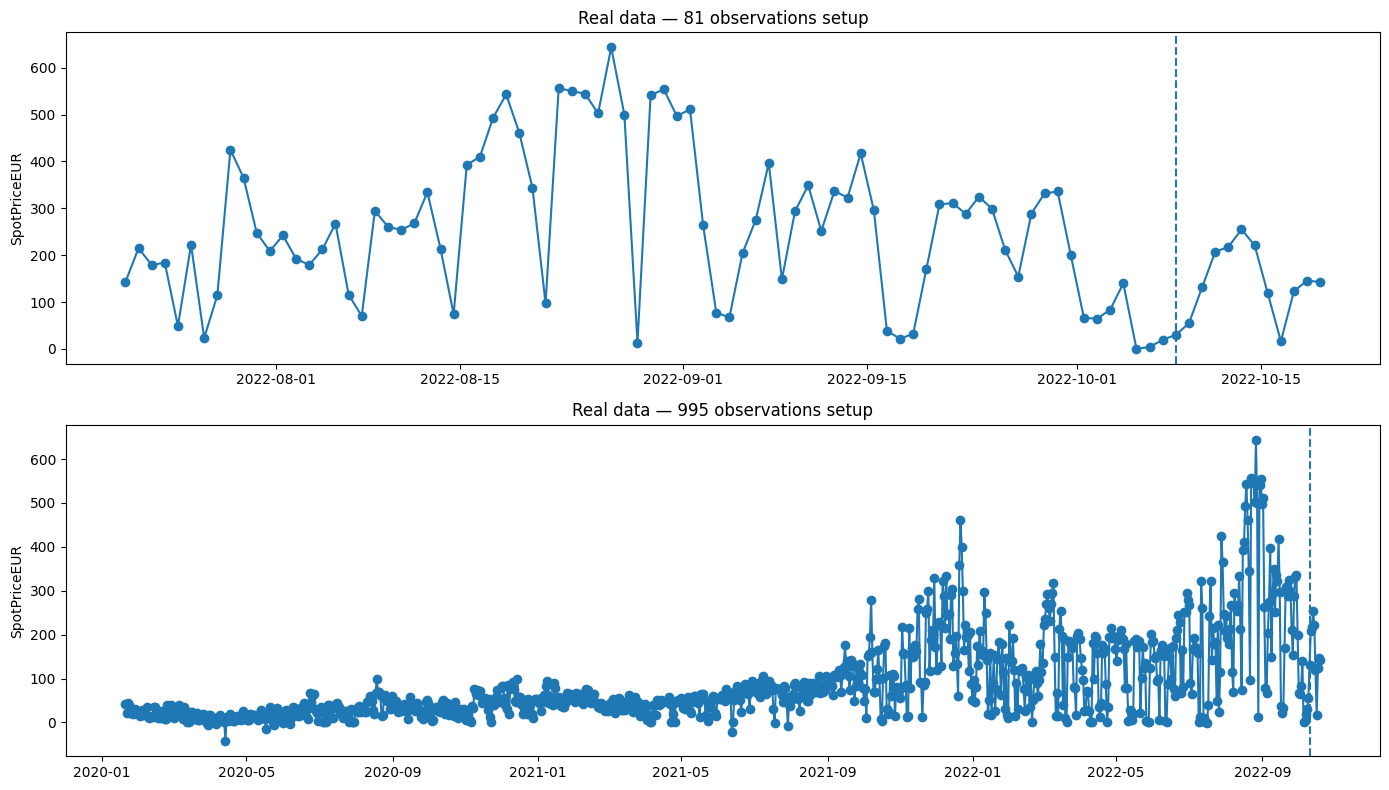}
    \caption{Data windows used in the numerical experiments on real data: the upper panel corresponds to the training window of length 81, while the lower panel corresponds to the training window of length 995.}
    \label{fig:real-data-windows}
\end{figure}

\begin{table}[H]
\centering
\begin{tabular}{lcccc}
\hline
Case & Hidden layers & Activations & Optimizer & Epochs \\
\hline
Gaussian, 81  & $(20,50,20)$ & swish / softplus / GELU & RMSprop & 1000 \\
Laplace, 81   & $(20,50,20)$ & swish / softplus / GELU & RMSprop & 1000 \\
Gaussian, 995 & $(25,50,25)$ & swish / softplus / GELU & RMSprop & 1000 \\
Laplace, 995  & $(25,50,25)$ & swish / softplus / GELU & RMSprop & 1000 \\
\hline
\end{tabular}
\caption{Network architectures used in the real-data experiments. The architectures were selected manually and then kept fixed within each sample-size setting.}
\label{tab:real-arch}
\end{table}

\begin{figure}[H]
    \centering
    \subfloat[Gaussian noise -- $c(t)$]{
        \includegraphics[width=0.32\linewidth]{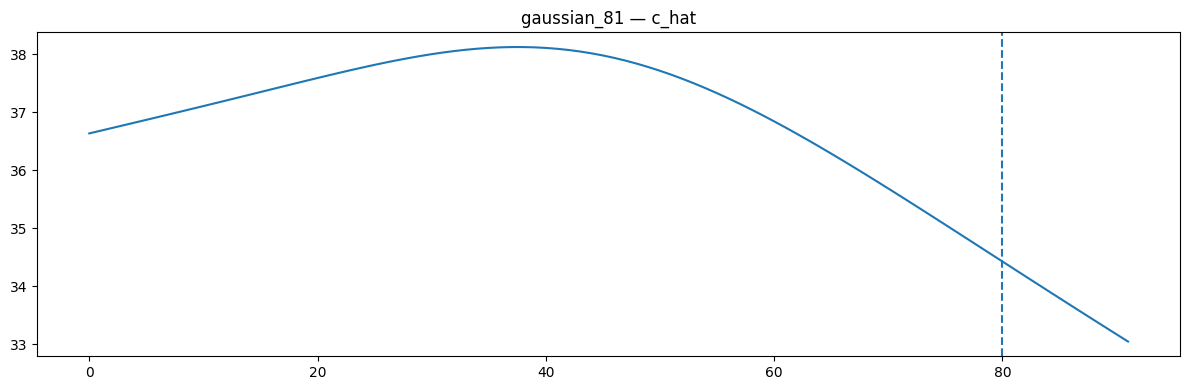}
    }
    \subfloat[Gaussian noise -- $\phi(t)$]{
        \includegraphics[width=0.32\linewidth]{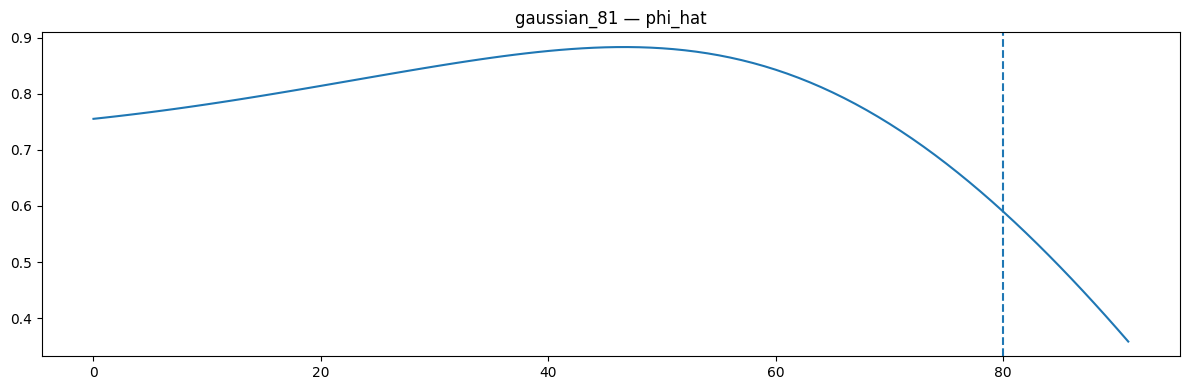}
    }
    \subfloat[Gaussian noise -- $\sigma^2(t)$]{
        \includegraphics[width=0.32\linewidth]{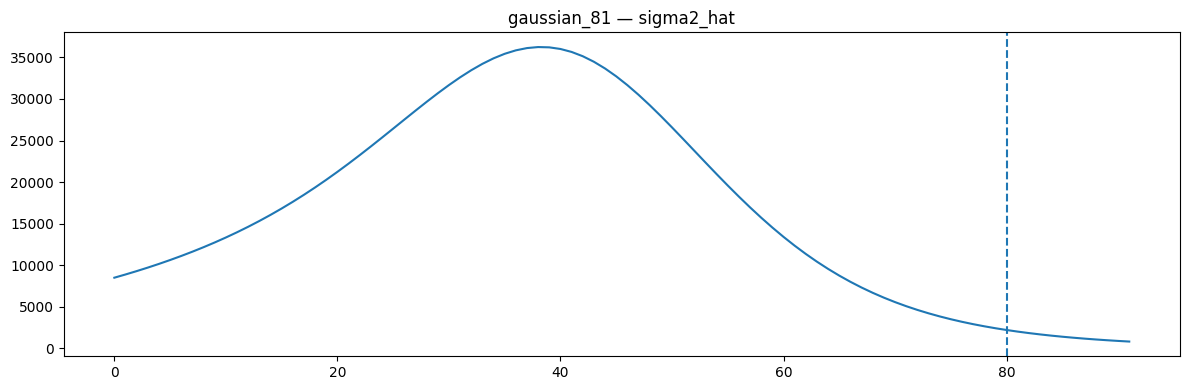}
    }

    \subfloat[Laplace noise -- $c(t)$]{
        \includegraphics[width=0.32\linewidth]{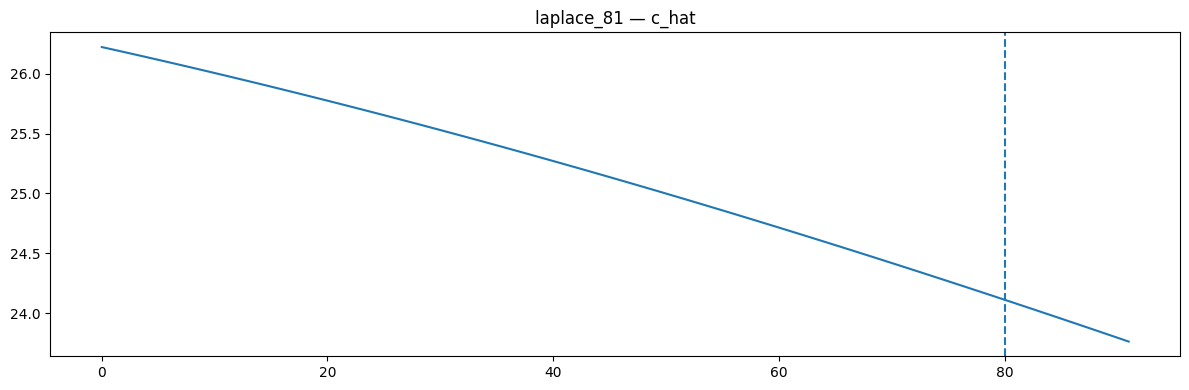}
    }
    \subfloat[Laplace noise -- $\phi(t)$]{
        \includegraphics[width=0.32\linewidth]{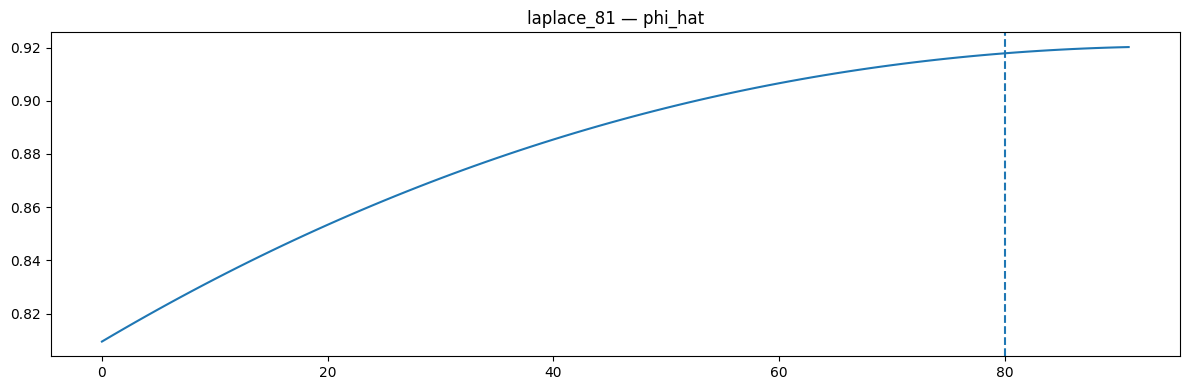}
    }
    \subfloat[Laplace noise -- $b(t)$]{
        \includegraphics[width=0.32\linewidth]{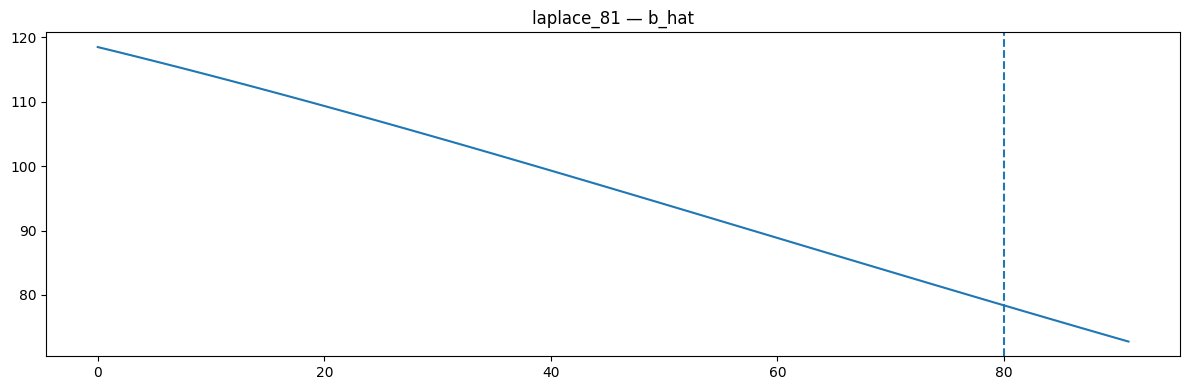}
    }
    \caption{Estimated time-dependent parameters for the real-data experiment with training set consisting of 81 observations.}
    \label{fig:real-81-params}
\end{figure}

\begin{figure}[H]
    \centering
    \subfloat[Gaussian noise -- $c(t)$]{
        \includegraphics[width=0.32\linewidth]{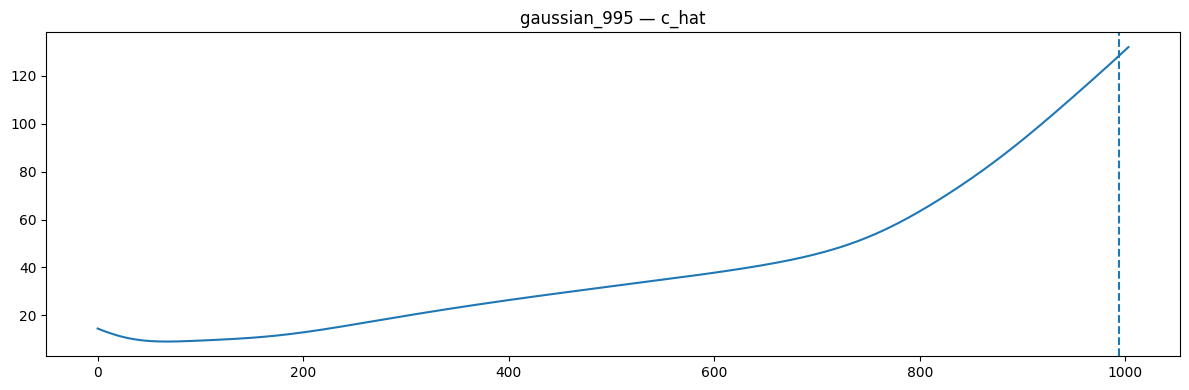}
    }
    \subfloat[Gaussian noise -- $\phi(t)$]{
        \includegraphics[width=0.32\linewidth]{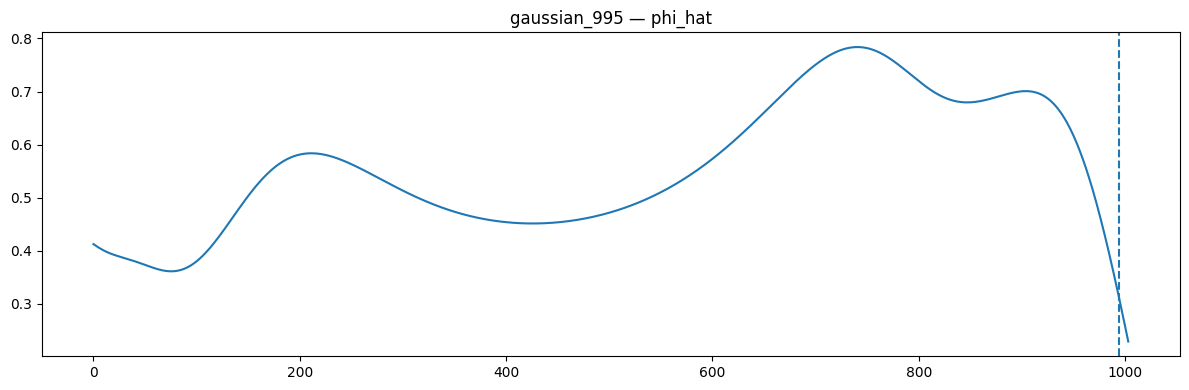}
    }
    \subfloat[Gaussian noise -- $\sigma^2(t)$]{
        \includegraphics[width=0.32\linewidth]{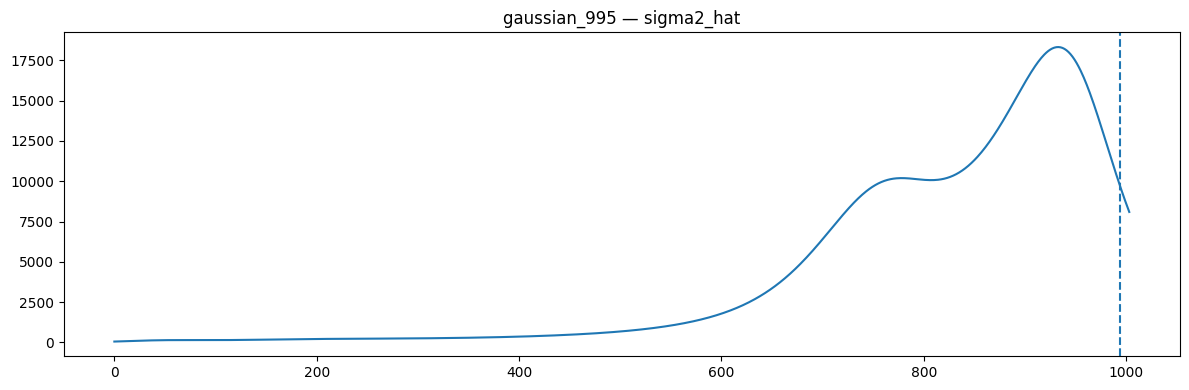}
    }

    \subfloat[Laplace noise -- $c(t)$]{
        \includegraphics[width=0.32\linewidth]{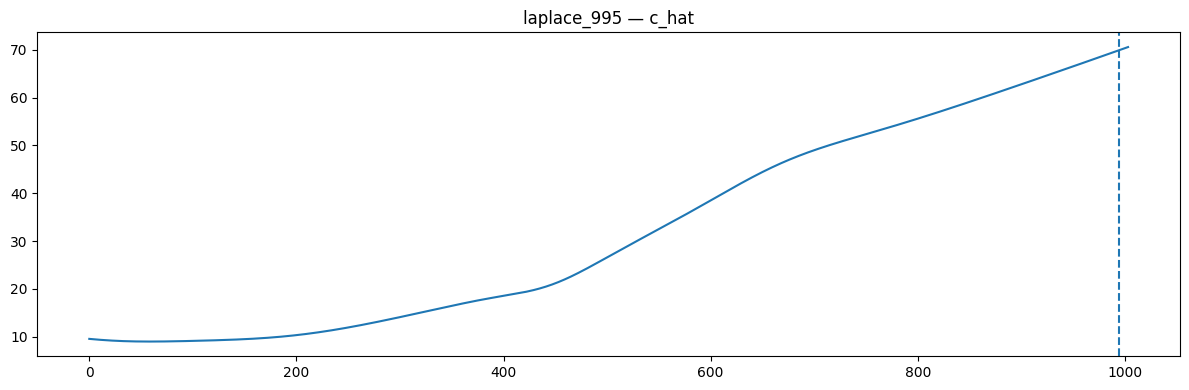}
    }
    \subfloat[Laplace noise -- $\phi(t)$]{
        \includegraphics[width=0.32\linewidth]{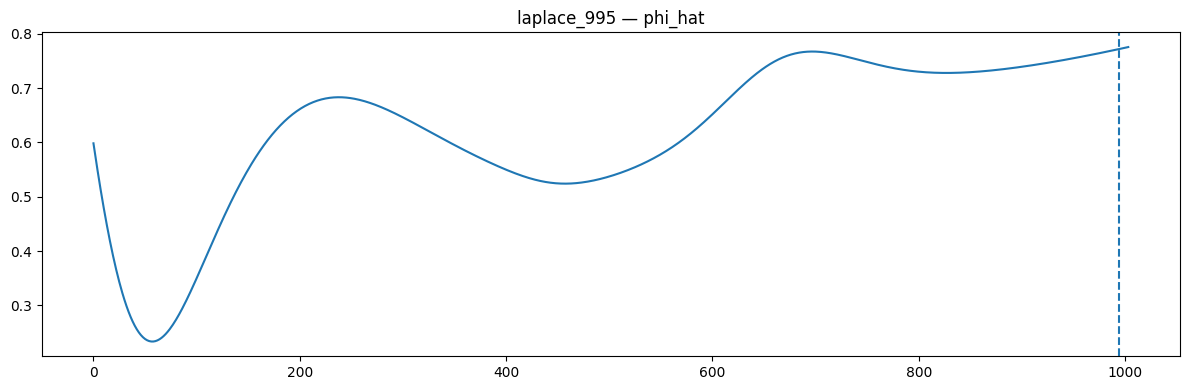}
    }
    \subfloat[Laplace noise -- $b(t)$]{
        \includegraphics[width=0.32\linewidth]{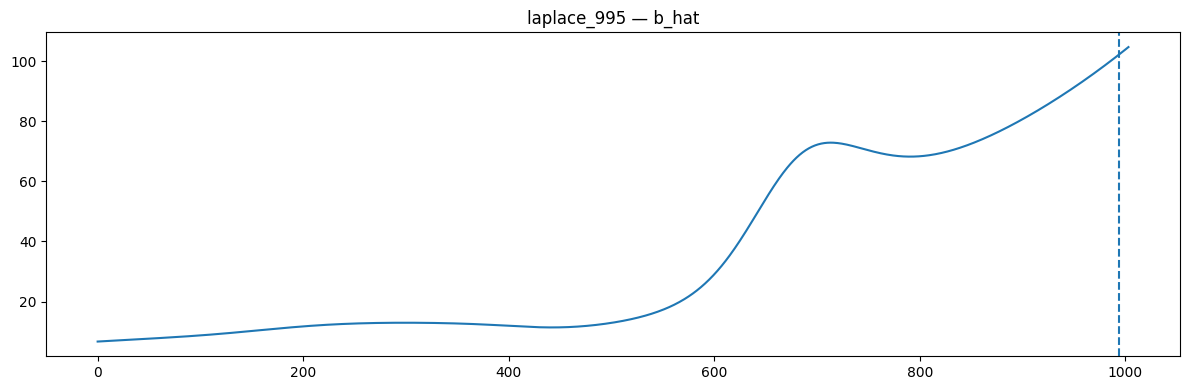}
    }
    \caption{Estimated time-dependent parameters for the real-data experiment with training set consisting of 995 observations.}
    \label{fig:real-995-params}
\end{figure}

To assess whether the fitted models reproduce the global behavior of the observed series, we additionally compare the observed trajectory with a trajectory simulated from the estimated time-dependent parameters. Since the longer training window provides a more informative setting, this comparison is reported only for the case of 995 observations; see Figure~\ref{fig:real-995-fit}.

\begin{figure}[H]
    \centering
    \subfloat[Gaussian model -- actual vs simulated]{
        \includegraphics[width=0.8\linewidth]{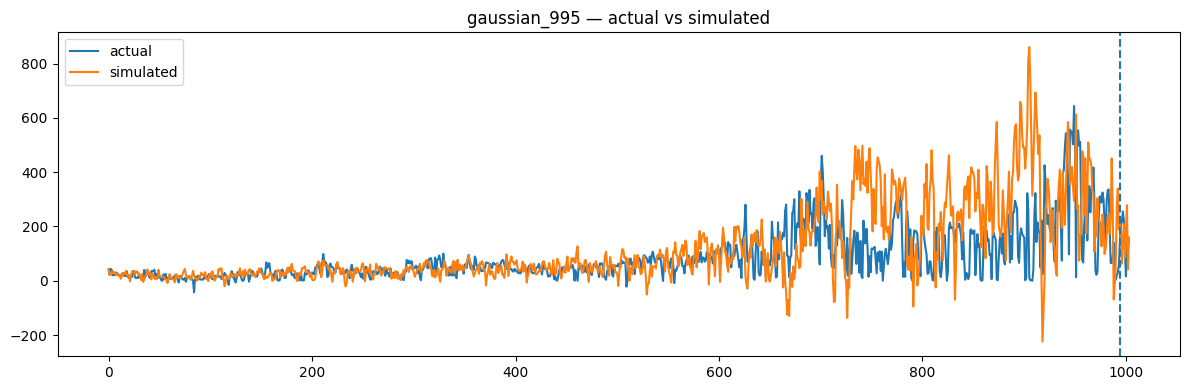}
    }

    \subfloat[Laplace model -- actual vs simulated]{
        \includegraphics[width=0.8\linewidth]{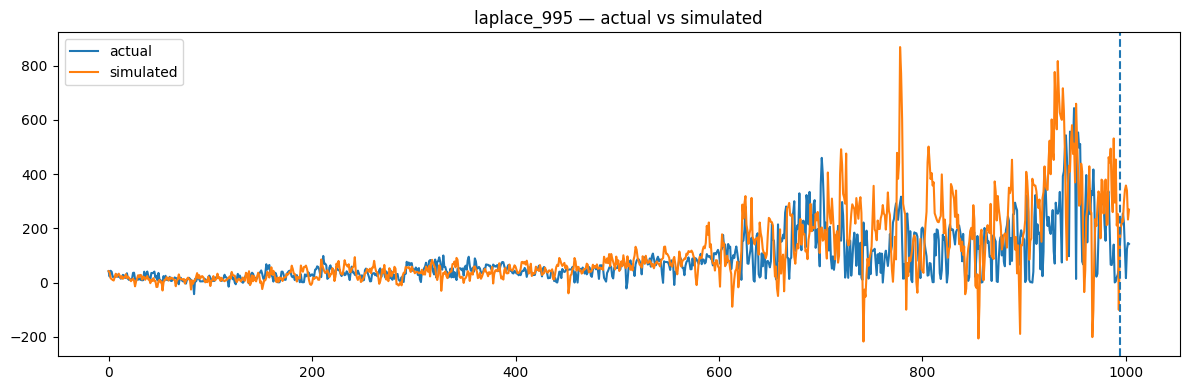}
    }
    \caption{Observed and simulated trajectories
    for the real-data experiment with training set consisting of 995 observations.}
    \label{fig:real-995-fit}
\end{figure}

Finally, Table~\ref{tab:real-forecast-errors} reports the absolute and relative forecast errors for one-step-ahead and two-step-ahead predictions. The relative error for horizon $h\in\{1,2\}$ is computed as
\[
\mathrm{rel.err.}_{N+h}=\frac{|y_{N+h}-\bar y_{N+h}|}{|y_{N+h}| + 10^{-12}} .
\]

\begin{table}[H]
\centering
\begin{tabular}{lccccc}
\hline
Case & Noise & 1-step abs. & 1-step rel. & 2-step abs. & 2-step rel. \\
\hline
81 observations  & Gaussian & 3.35   & 0.0612 & 68.62  & 0.5232 \\
81 observations  & Laplace  & 3.18   & 0.0581 & 59.75  & 0.4555 \\
995 observations & Gaussian & 38.83  & 0.1874 & 39.03  & 0.1794 \\
995 observations & Laplace  & 35.96  & 0.1735 & 15.16  & 0.0697 \\
\hline
\end{tabular}
\caption{Absolute and relative forecasting errors for one-step-ahead and two-step-ahead predictions in the real-data experiments.}
\label{tab:real-forecast-errors}
\end{table}




For the particular train--test splits considered here, the results suggest that the choice of the training-window length has a more pronounced effect on forecast accuracy than the assumed noise distribution. The comparison between Gaussian and Laplace noise should therefore be interpreted as illustrative and specific to these splits. A comprehensive forecasting evaluation would require a rolling or expanding-window backtest, together with comparisons against standard baseline models.

\FloatBarrier

\section{Key implementation details for the numerical experiments}

In this section, we reproduce the main implementation components used in the numerical experiments. The code follows the likelihood-based formulation introduced in Section~\ref{Section: model and loss} and is kept consistent with the publicly available notebooks. The losses below are written for ordered contiguous observations: when evaluated on the full ordered series, they coincide with the conditional Gaussian or Laplace negative log-likelihoods, up to constants. When smaller ordered mini-batches are used, the same formulas are optimized on contiguous sub-blocks and should be understood as a mini-batch approximation to the full conditional likelihood.
In particular, the temporal ordering of the observations is preserved
throughout the optimization procedure.

\subsection{Gaussian loss}

In the Gaussian case, the network outputs the triplet
\[
\Bigl(\hat c(t),\,\hat\phi(t),\,\ln\left(\hat\sigma^2(t)\right)\Bigr),
\]
and the loss is given by the Gaussian negative log-likelihood, up to an additive constant. The following implementation was used:
\begin{lstlisting}[language=Python, caption={Gaussian loss used in the synthetic and real-data experiments.}]
def gaussian_loss3(y_true, y_pred):
    N = tf.shape(y_true)[0] - 1
    y_true_1 = tf.slice(y_true, [1, 0], [N, 1])
    y_true_2 = tf.slice(y_true, [0, 0], [N, 1])

    c = tf.slice(y_pred, [1, 0], [N, 1])
    phi = tf.slice(y_pred, [1, 1], [N, 1])
    ln_sigma_squared = tf.slice(y_pred, [1, 2], [N, 1])
    sigma_squared = tf.exp(ln_sigma_squared)

    a = tf.square(y_true_1 - c - phi * y_true_2)
    b = a / sigma_squared

    loss = tf.reduce_sum(ln_sigma_squared + b, axis=0)
    return loss / 2.0
\end{lstlisting}

For a fully ordered batch, this is precisely the conditional Gaussian likelihood for the TVAR(1) model, written in terms of the network outputs. For smaller ordered mini-batches, it provides the corresponding blockwise approximation.

\subsection{Laplace loss}

In the Laplace case, the network outputs the triplet
\[
\Bigl(\hat c(t),\,\hat\phi(t),\,\ln\left(\hat b(t)\right)\Bigr),
\]
and the loss corresponds to the Laplace negative log-likelihood:
\begin{lstlisting}[language=Python, caption={Laplace loss used in the 
real-data experiments.}, label={lst:laplace-loss-real}]
def laplace_loss3(y_true, y_pred):
    N = tf.shape(y_true)[0] - 1
    y_true_1 = tf.slice(y_true, [1, 0], [N, 1])
    y_true_2 = tf.slice(y_true, [0, 0], [N, 1])

    c = tf.slice(y_pred, [1, 0], [N, 1])
    phi = tf.slice(y_pred, [1, 1], [N, 1])
    ln_b = tf.slice(y_pred, [1, 2], [N, 1])
    b = tf.exp(ln_b)

    a1 = tf.abs(y_true_1 - c - phi * y_true_2)
    a2 = a1 / b

    loss = tf.reduce_sum(ln_b + a2, axis=0)
    return loss
\end{lstlisting}

Hence, the two noise specifications are treated analogously, with the only difference being the assumed conditional distribution of the innovation term.

For experiments on real-world data, the TensorFlow implementation shown
in Listing \ref{lst:laplace-loss-real} is used directly. 
For synthetic-data experiments, including grid search with multistart initialization, we instead use the following PyTorch implementation of the same Laplace negative log-likelihood, with additional numerical stabilization and an optional smoothness regularization term:
\begin{lstlisting}[language=Python, caption={Laplace loss used in the 
synthetic-data experiments.}, label={lst:laplace-loss-synthetic}]
def laplace_ar1_loss(y_true, y_pred, *, start_idx=1, lam=0.0, eps=1e-12):
    y_prev = y_true[start_idx - 1 : -1, 0]
    y_curr = y_true[start_idx:, 0]
    
    c = y_pred[start_idx:, 0]
    phi = y_pred[start_idx:, 1]
    log_b = y_pred[start_idx:, 2]
    b = torch.exp(log_b).clamp_min(eps)
    
    resid = y_curr - c - phi * y_prev
    base = torch.sum(torch.log(b) + torch.abs(resid) / b)
    
    if lam > 0.0:
        diffs = y_pred[start_idx:, :] - y_pred[start_idx - 1 : -1, :]
        base = base + lam * torch.mean(diffs**2)
    return base

\end{lstlisting}

\subsection{Network architecture}

The parameter functions are approximated by a feedforward neural network with three outputs corresponding to the time-dependent intercept, autoregressive coefficient, and scale parameter. A representative implementation is given below:
\begin{lstlisting}[language=Python, caption={Feedforward neural network architecture for parameterizing the time-dependent coefficients.}]
def build_model(hidden_units, activations):
    model = Sequential()
    model.add(Input(shape=(1,)))
    for h, act in zip(hidden_units, activations):
        model.add(Dense(h, activation=act))
    model.add(Dense(3, activation='linear'))
    return model
\end{lstlisting}

The final hidden-layer widths and activations were selected separately for each experiment, either by grid search or manual tuning, as described
above.

\subsection{Recursive forecasting}

For the real-data experiments, the fitted parameter trajectories are used directly in recursive forecasting. In the Gaussian case, both the forecast mean and the forecast variance are propagated recursively:
\begin{lstlisting}[language=Python, caption={Recursive Gaussian forecast used in the real-data experiments.}]
def gaussian_forecast_recursive(fitted_df, start, horizon, alpha):
    rows = []
    y_prev = float(fitted_df.iloc[start - 1]["y"])
    mean_prev = None
    var_prev = None
    q = float(norm.ppf((1.0 + alpha) / 2.0))

    for step in range(1, horizon + 1):
        idx = start + step - 1
        c = float(fitted_df.iloc[idx]["c_hat"])
        phi = float(fitted_df.iloc[idx]["phi_hat"])
        sigma2 = float(fitted_df.iloc[idx]["sigma2_hat"])

        if step == 1:
            mean = c + phi * y_prev
            var = sigma2
        else:
            mean = c + phi * float(mean_prev)
            var = phi * phi * float(var_prev) + sigma2

        rows.append({
            "step": step,
            "forecast_mean": mean,
            "interval_lower": mean - q * math.sqrt(var),
            "interval_upper": mean + q * math.sqrt(var),
            "actual_y": float(fitted_df.iloc[idx]["y"]),
        })
        mean_prev, var_prev = mean, var
    return pd.DataFrame(rows)
\end{lstlisting}

In the Laplace case, the forecast mean is propagated recursively as well, whereas the prediction interval is constructed using the one-step and two-step formulas derived in Section~3:
\begin{lstlisting}[language=Python, caption={Recursive Laplace forecast used in the real-data experiments.}]
def laplace_forecast_recursive(fitted_df, start, horizon, alpha):
    rows = []
    y_prev = float(fitted_df.iloc[start - 1]["y"])
    means = []

    for step in range(1, horizon + 1):
        idx = start + step - 1
        c = float(fitted_df.iloc[idx]["c_hat"])
        phi = float(fitted_df.iloc[idx]["phi_hat"])
        b = float(fitted_df.iloc[idx]["b_hat"])

        if step == 1:
            mean = c + phi * y_prev
            radius = -b * math.log(1.0 - alpha)
        else:
            mean = c + phi * means[-1]
            if step == 2:
                a = abs(phi) * float(fitted_df.iloc[start]["b_hat"])
                radius = laplace_quantile_two_step(a=a, c=b, alpha=alpha)
            else:
                radius = np.nan

        rows.append({
            "step": step,
            "forecast_mean": mean,
            "interval_lower": mean - radius if np.isfinite(radius) else np.nan,
            "interval_upper": mean + radius if np.isfinite(radius) else np.nan,
            "actual_y": float(fitted_df.iloc[idx]["y"]),
        })
        means.append(mean)
    return pd.DataFrame(rows)
\end{lstlisting}

\subsection{Learning schedule and ordered training}

In the real-data experiments, the constant learning rate was replaced with a decreasing learning rate, which led to improved results. Moreover, the optimization was performed with \texttt{shuffle=False}, so that every mini-batch contains contiguous observations in chronological order:
\begin{lstlisting}[language=Python, caption={Training protocol used in the final experiments.}]
def make_lr_schedule(C=1000, scale=1.0):
    def schedule(epoch):
        return scale / (epoch + C)
    return schedule

C = 1000 if noise == "laplace" else 100
callbacks = [LearningRateScheduler(make_lr_schedule(C=C, scale=1.0))]

batch_size = len(x_train) if noise == "gaussian" and data_key == "81" else 16

history = model.fit(
    x_train,
    y_train,
    batch_size=batch_size, 
    epochs=1000,
    verbose=0,
    shuffle=False,
    callbacks=callbacks
)
\end{lstlisting}
When \texttt{batch\_size=len(x\_train)}, the loss is the full ordered conditional likelihood. When \texttt{batch\_size=16}, the same loss is evaluated on ordered contiguous sub-blocks and therefore acts as a mini-batch approximation to the full-series likelihood; transitions crossing mini-batch boundaries are not included in a given mini-batch evaluation.
\FloatBarrier

\FloatBarrier

\section{Conclusion and future work}

We proposed a neural-network-based likelihood framework for estimating time-dependent parameters in autoregressive models. The approach combines the interpretability of a structured stochastic model with the flexibility of data-driven function approximation, allowing relatively simple AR models to capture complex nonstationary behavior.

The analysis covered both Gaussian and Laplace noise specifications, together with recursive prediction formulas for the TVAR(1) case. In particular, 
we derived explicit prediction intervals in the Gaussian setting and tractable conditional error distributions in the Laplace setting. For $k = 1$, the prediction intervals admit closed-form expressions, whereas for $k \ge 2$ they do not admit closed-form solutions and are obtained as the unique root of a one-dimensional nonlinear equation. 

The numerical results on synthetic and real data indicate that the proposed method is capable of recovering meaningful parameter trajectories and producing interpretable illustrative forecasts under different noise assumptions.

Possible directions for future work include considering more general noise families such as the Normal--Laplace distribution, carrying out a broader rolling-window forecasting evaluation with standard baseline models, and studying multi-step prediction for general TVAR($p$) models. Another natural direction is the investigation of theoretical properties of the estimation procedure and its extension to multivariate autoregressive models.

\appendix

\section{Auxiliary calculations for the two-step-ahead prediction interval in the Laplace case}\label{appendixA}

In this appendix, we derive the distribution of the two-step-ahead prediction error in the TVAR(1) model with Laplace noise.

Let
\[
e_{N+2}
=
\hat \Phi(N+2)\hat b(N+1)\xi_{N+1}
+
\hat b(N+2)\xi_{N+2},
\]
where
\[
\xi_{N+1},\xi_{N+2}\stackrel{\text{i.i.d.}}{\sim}\mathrm{Laplace}(0,1).
\]
Set
\[
a:=|\hat \Phi(N+2)|\hat b(N+1),
\qquad
c:=\hat b(N+2).
\]
Then, conditionally on the available information up to time \(N\),
\[
e_{N+2}\overset{d}=X_1+X_2,
\]
where
\[
X_1\sim \mathrm{Laplace}(0,a),
\qquad
X_2\sim \mathrm{Laplace}(0,c),
\]
and \(X_1,X_2\) are independent.

\subsection{Characteristic function and density}

For a random variable \(X\sim \mathrm{Laplace}(0,s)\), the characteristic function is
\[
\varphi_X(t)=\frac{1}{1+s^2 t^2}.
\]
Hence,
\[
\varphi_{e_{N+2}}(t)
=
\varphi_{X_1}(t)\varphi_{X_2}(t)
=
\frac{1}{(1+a^2 t^2)(1+c^2 t^2)}.
\]

Assume first that \(a\neq c\). We use the decomposition
\[
\frac{1}{(1+a^2 t^2)(1+c^2 t^2)}
=
\frac{A}{1+a^2 t^2}
+
\frac{B}{1+c^2 t^2}.
\]
Matching coefficients gives
\[
A=\frac{a^2}{a^2-c^2},
\qquad
B=-\frac{c^2}{a^2-c^2}.
\]
Therefore,
\[
\varphi_{e_{N+2}}(t)
=
\frac{a^2}{a^2-c^2}\cdot \frac{1}{1+a^2 t^2}
-
\frac{c^2}{a^2-c^2}\cdot \frac{1}{1+c^2 t^2}.
\]

Using the inverse Fourier transform and the identity
\[
\mathcal{F}^{-1}\!\left(\frac{1}{1+s^2 t^2}\right)(x)
=
\frac{1}{2s}e^{-|x|/s},
\]
we obtain
\[
f_{e_{N+2}}(x)
=
\frac{a}{2(a^2-c^2)}e^{-|x|/a}
-
\frac{c}{2(a^2-c^2)}e^{-|x|/c},
\qquad x\in\mathbb R.
\]
Equivalently,
\[
f_{e_{N+2}}(x)
=
\frac{a\,e^{-|x|/a}-c\,e^{-|x|/c}}{2(a^2-c^2)}.
\]

\subsection{\texorpdfstring{The case $a=c$}{The case a=c}}

If \texorpdfstring{$a=c$}{a=c}, we take the limit in the previous expression:
\[
f_{e_{N+2}}(x)
=
\lim_{u\to a}
\frac{u\,e^{-|x|/u}-a\,e^{-|x|/a}}{2(u^2-a^2)}. \]
Writing
\[
g(u):=u\,e^{-|x|/u},
\]
we have
\[
g'(u)
=
e^{-|x|/u}\left(1+\frac{|x|}{u}\right).
\]
Hence,
\[
f_{e_{N+2}}(x)
=
\frac{1}{4a}e^{-|x|/a}\left(1+\frac{|x|}{a}\right),
\qquad x\in\mathbb R.
\]

\subsection{Distribution function and quantile equation}

Assume first that \(a\neq c\). Since the density is symmetric, for \(t\ge 0\) we obtain
\[
F_{e_{N+2}}(t)
=
1-\frac{a^2}{2(a^2-c^2)}e^{-t/a}
+\frac{c^2}{2(a^2-c^2)}e^{-t/c},
\]
while for \(t<0\),
\[
F_{e_{N+2}}(t)
=
\frac{a^2}{2(a^2-c^2)}e^{t/a}
-
\frac{c^2}{2(a^2-c^2)}e^{t/c}.
\]

Thus, for \(q\ge 0\),
\[
\mathbb P(|e_{N+2}|\le q)
=
F_{e_{N+2}}(q)-F_{e_{N+2}}(-q)
=
1-\frac{a^2e^{-q/a}-c^2e^{-q/c}}{a^2-c^2}.
\]
Therefore, the radius \(q_\alpha^{(2)}\) of the two-step-ahead prediction interval is determined by
\[
a^2e^{-q_\alpha^{(2)}/a}
-
c^2e^{-q_\alpha^{(2)}/c}
=
(1-\alpha)(a^2-c^2).
\]

In the case \(a=c\), the density is
\[
f_{e_{N+2}}(x)=\frac{1}{4a}e^{-|x|/a}\left(1+\frac{|x|}{a}\right),
\]
which yields
\[
\mathbb P(|e_{N+2}|\le q)
=
1-e^{-q/a}\left(1+\frac{q}{2a}\right),
\qquad q\ge 0.
\]
Hence, \(q_\alpha^{(2)}\) satisfies
\[
e^{-q_\alpha^{(2)}/a}\left(1+\frac{q_\alpha^{(2)}}{2a}\right)=1-\alpha.
\]

\section{Auxiliary lemmas for multi-step prediction intervals under Laplace innovations}\label{appendixB}
\begin{lemma}
\label{lem:cf_laplace}
Let $\xi\sim\mathrm{Laplace}(0,1)$, i.e.
$f_{\xi}(x)=\tfrac12 e^{-|x|}$.
Then, for all $t\in\mathbb{R}$,
\begin{equation}
  \E\big[e^{it\xi}\big]=\frac{1}{1+t^2}.
  \label{eq:cf_laplace}
\end{equation}
More generally, if $s>0$ and $X:=s\xi\sim\mathrm{Laplace}(0,s)$, then
$\E[e^{itX}]=1/(1+s^2 t^2)$.
\end{lemma}

\begin{proof}
The characteristic function of the Laplace distribution is given in Kotz et al. \cite[Section~2.1.2, Eq.~(2.1.8)]{KotzLaplace}.
\end{proof}
\begin{lemma}
\label{lem:partial_fraction}
Let $s_1,\dots,s_k>0$ be pairwise distinct. Define the rational function
\begin{equation*}
  R(u):=\prod_{j=1}^k \frac{1}{1+s_j^2 u},\qquad u\in\mathbb{C}\setminus\big\{-1/s_1^2,\dots,-1/s_k^2\big\}.
  \label{eq:R_u_def}
\end{equation*}
Then 
\begin{equation}
  R(u)=\sum_{j=1}^k \frac{A_j}{1+s_j^2 u},
  \label{eq:PF_decomp}
\end{equation}
where
\begin{equation}
  A_j=\prod_{\ell\in\{1,\dots,k\}\setminus\{j\}}\frac{s_j^2}{s_j^2-s_\ell^2},\qquad j\in\{1, \dots, k\},
  \label{eq:Aj_explicit}
\end{equation}
and $\sum_{j=1}^k A_j=1$.
\end{lemma}

\begin{proof}
Since the $s_j$ are pairwise distinct, the numbers $-1/s_j^2$ are pairwise distinct and hence
$R(u)$ has only simple poles at $u=-1/s_j^2$.
Therefore $R(u)$ admits a partial fraction decomposition with simple terms of the form
$A_j/(1+s_j^2 u)$.
Uniqueness follows because the functions $u\mapsto 1/(1+s_j^2 u)$ have distinct poles, hence they are linearly independent over $\mathbb{C}$.

\medskip
Multiply both sides of \eqref{eq:PF_decomp} by $(1+s_j^2 u)$ and then evaluate at the pole location $u=-1/s_j^2$.
On the left-hand side,
\begin{align*}
  \lim_{u\to -1/s_j^2} (1+s_j^2 u)R(u)
  &= \lim_{u\to -1/s_j^2}
     \prod_{\ell\ne j}\frac{1}{1+s_\ell^2 u}
  = \prod_{\ell\ne j}\frac{1}{1-s_\ell^2/s_j^2}
  = \prod_{\ell\ne j}\frac{s_j^2}{s_j^2-s_\ell^2}.
\end{align*}
On the right-hand side of \eqref{eq:PF_decomp}, we obtain
\begin{align*}
  \lim_{u\to -1/s_j^2} (1+s_j^2 u)\sum_{m=1}^k \frac{A_m}{1+s_m^2 u}
  &= \lim_{u\to -1/s_j^2}
     \Bigg(A_j + \sum_{m\ne j} A_m\,\frac{1+s_j^2 u}{1+s_m^2 u}\Bigg)
  = A_j,
\end{align*}
because for $m\ne j$ the factor $(1+s_j^2 u)$ vanishes at $u=-1/s_j^2$ while the denominator remains nonzero.
Equating the two limits yields \eqref{eq:Aj_explicit}.

\medskip
Set $u=0$ in \eqref{eq:PF_decomp}. Since $R(0)=1$, we obtain
$1=\sum_{j=1}^k A_j$.
\end{proof}
The following result is well-known.
\begin{lemma}[Inverse Fourier transform of $(1+s^2 t^2)^{-1}$]
\label{lem:inv_fourier}
Let $s>0$. Define
\begin{equation}
  g_s(x):=\frac{1}{2\pi}\int_{-\infty}^{\infty}\frac{e^{-itx}}{1+s^2 t^2}\,\mathrm{d}t,
  \qquad x\in\mathbb{R}.
  \label{eq:gs_def}
\end{equation}
Then
\begin{equation}
  g_s(x)=\frac{1}{2s}\exp\Big(-\frac{|x|}{s}\Big),\qquad x\in\mathbb{R}.
  \label{eq:gs_closed}
\end{equation}
In particular, $g_s$ is the density of $\mathrm{Laplace}(0,s)$.
\end{lemma}






\bibliographystyle{siam}
\bibliography{tvar_1}

\end{document}